\documentclass[accepted]{uai2026} 
                        
\usepackage[american]{babel}

\usepackage{natbib} 
\usepackage{mathtools} 
\usepackage{amsfonts} 
\usepackage{amsthm}
\usepackage{booktabs} 
\usepackage{tikz} 

\hypersetup{colorlinks=false,pdfborder={0 0 0}}

\title{Partially Observed Structural Causal Models}

\author[1,2]{\href{mailto:<turan.orujlu@tuebingen.mpg.de>?Subject=Your UAI 2026 paper}{Turan Orujlu}{}}
\author[3]{Jordan Matelsky}
\author[1]{Martin V. Butz}
\author[2,4,5]{Charley M. Wu}
\author[3]{Konrad P. Kording}
\affil[1]{%
    University of T{\"u}bingen
}
\affil[2]{%
    MPI for Biological Cybernetics
}
\affil[3]{%
    University of Pennsylvania
  }
\affil[4]{%
    TU Darmstadt
  }
\affil[5]{%
    Hessian.AI
  }
\newtheorem{theorem}{Theorem}[section]
\newtheorem{remark}[theorem]{Remark}
\newtheorem{definition}[theorem]{Definition}
\newtheorem{proposition}[theorem]{Proposition}
\newtheorem{lemma}[theorem]{Lemma}
\newtheorem{assumption}[theorem]{Assumption}

\newcommand{\R}{\mathbb{R}}

\renewcommand{\P}{\mathbb{P}}

\newcommand{\Bcal}{\mathcal{B}}
\newcommand{\Fcal}{\mathcal{F}}

\newcommand{\Vcal}{\mathcal{V}}

\newcommand{\Ocal}{\mathcal{O}}

\newcommand{\betatilde}{\tilde{\beta}}
\newcommand{\Vtilde}{\tilde{V}}
\newcommand{\Pa}{\text{Pa}}
\newcommand{\doop}{\text{do}}

\newcommand{\cH}{\mathcal{H}}

\begin{document}
\maketitle

\begin{abstract}
  Here we introduce Partially Observed Structural Causal Models (POSCMs) as an extension of structural causal models (SCMs) to settings where upstream contexts co-determine both the interaction structure and downstream mechanisms on observed variables. POSCMs thus provide a self-contained causal modeling framework for endogenous graphs, allowing for an intervention hierarchy spanning node- and edge-level contexts and endogenous variable interventions. To define edge interventions, we separate node mechanisms into edge-local transmission channels that can be modified without changing the source node or the rest of the target mechanism. We provide an identifiability theory that clarifies which intervention families would suffice to disentangle structure formation from mechanisms. We then empirically validate these theoretical results in two external simulators: a biophysically detailed virtual human retina and a gene-regulatory analogue. The experiments reproduce non-identifiability under latent context, expose  structure-mechanism confounding under latent edges, and recover pathway-level input-output relationships under targeted interventions, consistent with our positive Markov kernel identifiability results. Together, POSCMs provide an intervention-oriented framework for causal systems in which contexts, graph structure, mechanisms, and measurements are jointly generated and only partially observed. 
\end{abstract}

\section{Introduction}\label{sec:intro}

Many biological, social, and engineered systems generate networks whose structure depends on upstream contexts, which in turn, shape downstream dynamics on nodes. Structural Causal Models (SCMs) typically treat graphs as fixed/exogenous \citep{pearl2009, Peters2017}. However, in many real-world systems, whether one variable causally influences another is not a fixed background fact, but is itself determined by upstream processes and can be altered by interventions. A genetic perturbation may abolish a synapse between two neurons. A policy change may sever a regulatory link between two companies. In such settings, we need causal semantics in which the existence of edges is itself an endogenous mechanism, allowing them to be reasoned about, intervened on, and identified, rather than assuming a static scaffold on which mechanisms are defined. Furthermore, causal discovery methods based on SCMs typically assume we can directly observe system variables, potentially based on some knowledge of the network, and infer a single, context-independent mechanism for each variable. While powerful, the assumptions behind classical SCMs often fail to capture the complexity of real-world systems.

Consider neuroscience, where we aim to infer neural circuitry from indirect, modality-specific measurements. For example, calcium imaging does not reveal latent neural activity directly: it reports a noisy fluorescence proxy filtered by
calcium dynamics, sensor expression, optical noise, and the sampled
field of view. Other readouts expose different parts of the same
circuit: electrophysiology records activity for selected cells without the full synaptic graph, whereas anatomical reconstruction can reveal wiring or morphology without simultaneous activity. Thus the available data are noisy measurements of a neuron's activity (mechanism) and connectivity (structure) that depend on its type and local environment (context), which are themselves often noisily observed. Crucially, this neuronal type is often shaped by the connectivity itself during development, creating an entanglement between context and structure. We thus need to solve a nontrivial inference about context, mechanism, and structure, where each is partially observed. Similarly, the transition from Markov decision processes (MDPs) to partially observed MDPs (POMDPs) fundamentally changes the problem structure, requiring new solution concepts once the agent no longer has direct access to the full state \citep{KAELBLING199899}. Here, causal modeling faces an analogous gap when structure, context, and mechanisms are only partially observed. Standard SCMs correspond to the ``fully observed'' regime, assuming a fixed, known graph with directly measurable variables. Thus, we require a partially observed counterpart for describing more real-world causal systems.

\begin{figure}[t]
    \centering
    \includegraphics[width=1.0\linewidth]{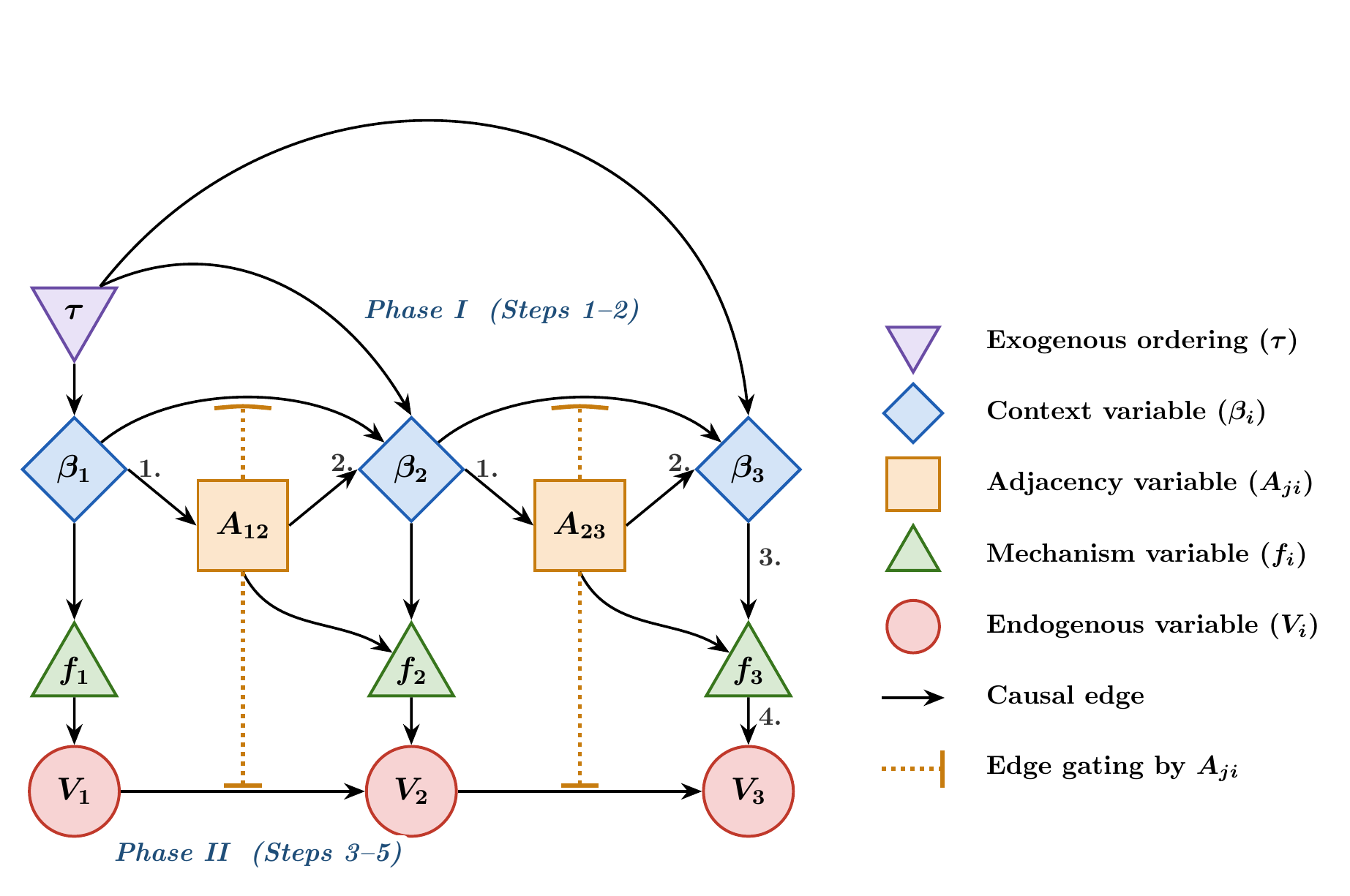}
    \caption{Expanded DAG for a 3-node POSCM under ordered generation $\tau = (1,2,3)$. Labels 1--4 mark (1) structure formation, (2) context generation, (3) mechanism assignment, and (4) endogenous-value generation. Squares $A_{ji}$ are variables in the augmented DAG because the graph is generated aleatorically by the structure kernel rather than fixed in advance. A dashed gate labelled by $A_{ji}$ indicates that the $j\to i$ value channel is active only when $A_{ji}=1$; when $A_{ji}=0$, the corresponding channel contribution is zero. Solid arrows are causal arrows in the augmented DAG, whereas the induced $V$-edges are the realized value channels selected by $A$.}
 \label{fig:poscm_hierarchy}
\end{figure}

Here we introduce \emph{Partially Observed Structural Causal Models (POSCMs)} as a formal framework to model such cases: upstream contexts $\beta$ parametrize both (i) the stochastic mechanism that forms directed edges $A$ and (ii) the value mechanisms that generate observed variables $V$ along the realized graph. To our knowledge, no prior causal framework treats structure formation as an endogenous causal mechanism driven by upstream contexts while also providing an explicit intervention vocabulary that separates changes to context, structure, and value mechanisms. Moreover, none embeds these capabilities within an ordered-generation semantics that accommodates partial observability of the generating context variables.

\paragraph{Contributions.}
We (i) formalize POSCMs via ordered generation, (ii) show that SCMs arise as a special case, (iii) use a message-augmented POSCMs to define edge interventions (iv) provide positive Markov kernel identifiability results and barriers under certain readout and coverage assumptions, and (v) validate these theoretical results in a virtual retina simulator.

For related work, see Appendix~\ref{sec:related_work}.

\section{POSCM formalism}
\label{sec:poscm}

Context affects causal structure (e.g., a cone photoreceptor's type determines how it processes information and whether it synapses onto an ON or OFF bipolar cell) , while causal structure in turn affects the generative model for context (the pattern of synaptic inputs a bipolar cell receives influences its functional identity as ON- or OFF-type). The relevant generative models are thus non-trivially entangled. To causally model such situations without introducing cycles (see Remark~\ref{rem:acyclicity-ordered-generation}), we introduce a formalism based on \textit{ordered generation} \citep{Simon1977, doi:10.1126/science.286.5439.509}, ensuring the overall generative process remains acyclic.

\begin{definition}[POSCM with Ordered Generation]
\label{def:poscm}
A POSCM is a tuple $M = (A, \beta, V, U, \alpha, \phi, f, \Gamma, \tau)$, characterized by an adjacency variable ($A$), context variables ($\beta$), endogenous variables ($V$), exogenous variables ($U$), a structure kernel ($\alpha$), context variable mechanisms ($\phi$), endogenous variable mechanisms ($f$), mechanism operator ($\Gamma$), and an exogenous total ordering $\tau$ over the nodes $V=\{V_1, \dots, V_N\}$.

The generative process proceeds sequentially according to the order $\tau$. Without loss of generality, assume $\tau(i)=i$.

\textbf{Phase I: Sequential Generation of Structure and Context}

For $i = 1$ to $N$:
\begin{enumerate}[leftmargin=*, noitemsep]
    \item \textbf{Structure Formation (Incoming Edges):} For each potential parent $j < i$, the edge $A_{ji}$ is drawn from a structure kernel whose conditional distribution depends on the (already generated) source context $\beta_j$:
    $A_{ji} = \alpha(\beta_j, U^{A}_{ji})$. Let $\Pa_A(i) = \{j : A_{ji}=1\}$ be the realized parent set.
    \item \textbf{Context Generation:} The context $\beta_i$ is generated based on the contexts of its realized parents:
    $\beta_i = \phi_i(\{\beta_j : j \in \Pa_A(i)\}, U^\beta_i)$.
\end{enumerate}

This interleaved process establishes a valid Directed Acyclic Graph (DAG) over the expanded set of variables (e.g., $\beta_1 \to A_{12} \to \beta_2 \to \dots$, see Fig.~\ref{fig:poscm_hierarchy}), resolving any circularity between $A$ and $\beta$ and capturing the co-evolution of structure and identity.

\textbf{Phase II: Mechanisms and Variables}

After $(A, \beta)$ are fully generated:
\begin{enumerate}[leftmargin=*, noitemsep, resume]
    \item \textbf{Mechanism Assignment:} The mechanism $f_i$ is drawn conditioned on the local context $\beta_i$ \emph{and} the realized parent set $\Pa_A(i)$. The parent set defines the input signature (domain) of the function.
    $f_i = \Gamma(\beta_i, \Pa_A(i), U^f_i)$.
    \item \textbf{Endogenous Variable Generation (SCM):} $V_i = f_i(\{V_j : j \in \Pa_A(i)\}, U^V_i)$.
\end{enumerate}

\end{definition}

$M$ thus specifies the \emph{data-generating causal process}.
What can be \emph{observed} is specified separately by a measurement model (Sec.~\ref{sec:measurement-model}), and what can be \emph{manipulated} is specified by an intervention family (Sec.~\ref{sec:interventions}). For each $i$, the vector $A_i=(A_{1i},\ldots,A_{i-1,i})$ encodes the realized parent set $\Pa_A(i)=\{j<i:A_{ji}=1\}$. Thus $A$ is not merely a drawing of the graph, but is a random variable in the augmented causal model, generated by the structure kernel. Saying that $\Pa_A(i)$ defines the \emph{input signature} of $f_i$ means that, after $A_i$ is realized, $f_i$ is a map from the realized parent-value space $\prod_{j\in\Pa_A(i)}\mathcal{V}_j$ and exogenous noise to $\mathcal{V}_i$. Equivalently, $V_i$ can depend functionally only on the variables $V_j$ for which $A_{ji}=1$. We write $\mathcal{B}_i$ and $\mathcal{V}_i$ for the context and endogenous-value spaces of node $i$, and use $\mathcal{B}$ and $\mathcal{V}$ when a common space is assumed for notational simplicity.

\subsection{Measurement model}
\label{sec:measurement-model}

The causal model $M$ does not, by itself, specify what is observed.
Instead, we couple $M$ to a (possibly noisy) measurement model $\Ocal$ that determines the available data.

\begin{definition}[Measurement model]
\label{def:measurement-model}
Let $M$ be a POSCM with generated variables $(A,\beta,V)$. A measurement model for $M$ is a tuple $\Ocal = (\Ocal^{A}, \Ocal^{\beta}, \Ocal^{V}, U^{\Ocal})$ specifying observation channels
\begin{align}
\tilde A &= \Ocal^{A}(A, U^{\Ocal}_A), &
\betatilde &= \Ocal^{\beta}(\beta, U^{\Ocal}_\beta), &
\Vtilde &= \Ocal^{V}(V, U^{\Ocal}_V),
\end{align}
where $U^{\Ocal}=(U^{\Ocal}_A,U^{\Ocal}_\beta,U^{\Ocal}_V)$ denotes measurement noise.
The observable $\tilde Z$ may include any subset of $(\tilde A,\betatilde,\Vtilde)$ (e.g., adjacency observed but contexts latent).
\end{definition}

\paragraph{Convention.}
Throughout, identifiability statements treat $\Ocal$ as fixed and known (or identifiable within a specified parametric family), and ask which parts of the causal model $M$ are determined from interventional distributions over $\tilde Z$.
This formalism highlights a fundamental entanglement: the mechanism cannot be defined independently of the structure, since the structure determines the mechanism's input domain.

\begin{remark}[Reduction to SCMs]
\label{rem:scm}
A POSCM reduces to a standard SCM variant (e.g., heterogeneous mechanism SCMs) iff the adjacency matrix variable $A$ and the mechanism variables $f$ are fixed to $A^*$ and $f^*$ respectively, i.e., $A \sim \delta_{A^*}, f \sim \delta_{f^*}$ where $\delta$ is the Dirac delta function.
\end{remark}

POSCMs specify an ordered stochastic process $\tau\to\beta\to(A,f)\to V$, where context-dependent kernels generate adjacency and mechanisms before values are generated. The randomness of $\alpha$ and $\Gamma$ is aleatoric: it describes variation across units, developmental outcomes, animals, or perturbational regimes, not merely uncertainty about one hidden fixed graph. A standard SCM over $V$ is appropriate when the graph and mechanisms are fixed, or when the target is a within-context effect on a measured adult network. A mixture of fixed-graph SCMs is recovered as a special case with discrete observed context and degenerate $\alpha,\Gamma$ within each component. POSCMs are useful when the scientific target is the measure by which context generates structure and mechanisms, e.g., $\mu_j^{\beta\text{-do}}(b)=\P(A_{j,>j}\mid\doop(\beta_j=b))$, and when data should be pooled across local, latent, noisy, or intervention-shifted contexts rather than enumerating whole-graph mixture components.

This viewpoint is also compatible with the independent  causal mechanisms intuition. In the common factorized case $\Gamma(f\mid A,\beta)=\prod_i \Gamma_i(f_i\mid \beta_i,\Pa_A(i))$, assigned mechanisms are independent across nodes conditional on their realized contexts and parent sets. They can nevertheless be marginally dependent because shared context $\beta$ co-determines both wiring $A$ and mechanism assignment $f$. Thus, autonomy lives at the level of the structural modules ($\alpha$, $\phi$, $\Gamma$) and their exogenous noises, not necessarily at the level of realized $(A,f)$. The SCM reduction above is the degenerate case $\alpha(\cdot\mid\beta)=\delta_{A^*}$ and $\Gamma_i(\cdot\mid\beta_i,\Pa_A(i))=\delta_{f_i^*}$, where there is no nontrivial stochastic dependence among assigned mechanisms; the substantive SCM assumption is modularity, i.e., interventions replace one structural equation while leaving the others fixed.

\begin{remark}[Acyclicity and ordered generation]
\label{rem:acyclicity-ordered-generation}
Ordered generation combines two established ideas: the topological-order forward sampling standard in acyclic SCMs, where each variable is generated as a function of earlier variables and exogenous noise \citep{Simon1977}, and the sequential node-arrival construction used in graph generative models \citep{doi:10.1126/science.286.5439.509}. It extends both by interleaving edge formation with context propagation under causal intervention semantics; the former assumes a fixed graph, while the latter lacks an intervention calculus.

The ordered-generation semantics assumes an exogenous total order $\tau$ over nodes, which rules out instantaneous feedback loops in the one-shot (non-temporal) model. This assumption is shared with \emph{DAG-SCMs} \citep{pearl2009} and most DAG-based causal discovery methods \citep{spirtes2000causation, chickering2002optimal, Peters2017},
and it ensures that the induced augmented causal graph over $(\beta,A,V)$ is acyclic
(Fig.~\ref{fig:poscm_hierarchy}).

Importantly, $\tau$ does not need to be interpreted as an arbitrary modeling artifact:
in settings with an intrinsic temporal or developmental sequence, $\tau$ can be anchored
to observed time. For example, in our virtual retina experiments (Sec.~\ref{sec:experiments}),
the laminar circuit provides a biologically grounded ordering (PR$\to$HZ$\to$BC$\to$AC$\to$RGC),
which we use as $\tau$ at the layer level.

When such an intrinsic order is unavailable, $\tau$ should be viewed as part of the POSCM
specification; different choices of $\tau$ may correspond to different models and thus
different counterfactual semantics, analogous to how different SCM mechanism specifications
can agree observationally yet disagree counterfactually.
\end{remark}

In the one-shot ordered model, edges are generated before Phase-II values. Allowing same-round dependence $A_{ji}\leftarrow V_j$ would entangle graph formation with value propagation unless temporal or cyclic semantics were added. POSCMs therefore place pre-Phase-II state variables in $\beta$: cell type, morphology, chromatin state, developmental history, past activity, etc. Activity-dependent connectivity can still be modeled by including past activity in $\beta$, or dynamically by time-unrolling, e.g. $A^{t+1}_{ji}\leftarrow V^t_j$.

\section{Message-augmented POSCM}
\label{sec:message-aug}

A \emph{message} is an edge-local transmission channel. In a neural circuit, $A_{ji}$ is the presence of a synapse, $H^V_{i\leftarrow j}(V_j)$ is the edge-local dendritic or synaptic transformation of the presynaptic activity, i.e., the message, and $F_i$ is the postsynaptic integration map. Thus, message augmentation is not introduced for its own sake: it is the additional causal structure needed to ask whether one dyadic channel $j\to i$ was changed while the source node $V_j$, the rest of $f_i$, and all other incoming channels remain fixed.

This places edge-message interventions between node interventions and arbitrary mechanism replacement. Node interventions change all outgoing uses of a source value, whereas an edge-message intervention changes only one channel. The decomposition is therefore part of the modeling interface: when messages correspond to experimentally addressable channels, edge-level counterfactuals are meaningful; when only the node-level kernel is meaningful, POSCM kernel targets remain well-defined without identifying a particular message gauge.

\begin{definition}[Message-augmented POSCM]
\label{def:msg-poscm}
Fix an ordering $\tau$ and, for each node $i$, let $m_i := |\{j: j<i\}|=i-1$ denote the number of \emph{potential} parents under ordered generation.
A POSCM is \emph{message-augmented} if for each node $i$ there exists a finite message dimension $d_i$ (fixed for that node) and functions
\begin{equation*}
    H^{\beta}_{i\leftarrow j}:\Bcal \to \R^{d_i},
    \qquad
    H^{V}_{i\leftarrow j}:\Vcal \to \R^{d_i},
\end{equation*}
together with aggregation maps
\begin{equation*}
    \Phi_i:\R^{d_i \times m_i}\to \mathcal{B},
    \qquad
    F_i:\R^{d_i\times m_i} \to \Vcal,
\end{equation*}
such that (with the convention $M_{i\leftarrow j}\equiv 0$ when $A_{ji}=0$)
\begin{align}
M^{\beta}_{i\leftarrow j} &= A_{ji}\,H^{\beta}_{i\leftarrow j}(\beta_j), &
\beta_i &= \Phi_i\!\left(\{M^{\beta}_{i\leftarrow j}\}_{j<i},\,U^\beta_i\right), \label{eq:msg-chi} \\
M^{V}_{i\leftarrow j} &= A_{ji}\,H^{V}_{i\leftarrow j}(V_j), &
V_i &= F_i\!\left(\{M^{V}_{i\leftarrow j}\}_{j<i},\,U^V_i\right). \label{eq:msg-X}
\end{align}
\end{definition}

\begin{remark}[Fixed message dimension]
The message codomain $d_i$ is \emph{fixed for node $i$} and does not vary with the realized parent set size $|\mathrm{pa}(i)|$.
This resolves the non-operational ``variable-dimension'' issue that arises if one defines an $i\leftarrow j$ intervention using a representation whose dimension depends on which other parents happen to be present in a given world.
\end{remark}

\paragraph{Existence via Kolmogorov--Arnold--Sprecher.}
Definition~\ref{def:msg-poscm} assumes a dyadic message parameterization. For broad continuous mechanisms on compact domains, such parameterizations exist by the Kolmogorov--Arnold--Sprecher (KAS) representation theorem; Appendix~\ref{app:kas_density} gives the construction and shows it is dense in Lebesgue space $L^p(\mu)$. We use KAS only as an existence result. The interventions in Sec.~\ref{sec:interventions} are defined directly on the message primitives $(H^\beta,H^V)$, not on a unique KAS coordinate system.

\begin{remark}[Representational non-uniqueness as a gauge]
\label{rem:kas-gauge}
KAS representations are generally non-unique: different internal representations of the same parent-to-child mechanism can induce the same overall mechanism.
Accordingly, identifiability of \emph{message coordinates} should be understood either (i) relative to a fixed message parameterization (Def.~\ref{def:msg-poscm}), or (ii) ``up to'' the induced representation gauge. A convenient way to formalize this is to view the gauge as the class of internal message reparameterizations that preserve the induced parent-to-child mechanism; for example, any bijection of the message space that is compensated inside $F_i$ (and similarly inside $\Phi_i$) yields an observationally equivalent representation.
For estimation, kernel targets such as $\P(V_i\mid V_{\Pa_A(i)},\beta_i,\Pa_A(i))$ are gauge invariant. In contrast, dyadic message estimates and edge-message interventions require either a fixed gauge, a physical channel anchoring, or reporting only gauge-invariant slot-response kernels.
\end{remark}

\section{Intervention Hierarchy}
\label{sec:interventions}

POSCMs admit interventions that act on context/value (i) nodes and (ii) edges. To make edge interventions operational, we define them on the \emph{message primitives} of the message-augmented representation (Def.~\ref{def:msg-poscm}).

\subsection{Primitive intervention types}

\begin{definition}[Node interventions]
\label{def:node-interventions}
A \emph{$\beta$-node intervention} $\doop(\beta_j= \widetilde b)$ is a Phase-I intervention: it replaces the structural equation for $\beta_j$ by the constant $\widetilde b$ and then continues the ordered generative process. Consequently, it may change downstream contexts, edge formation, and mechanism assignment through $\alpha,\phi$, and $\Gamma$.

A \emph{$V$-node intervention} $\doop(V_j=\widetilde v)$ is a Phase-II intervention: after $(A,\beta)$ and the mechanisms have been realized, it replaces the structural equation for $V_j$ by the constant $\widetilde v$ and propagates through the realized Phase-II SCM, leaving Phase-I quantities and assigned mechanisms unchanged.
\end{definition}

\begin{definition}[Edge-message interventions]
\label{def:edge-message-interventions}
In a message-augmented POSCM (Def.~\ref{def:msg-poscm}), a \emph{$\beta$-edge intervention} on dyad $s\to t$ replaces the context message function on that channel,
\begin{equation*}
    \doop_{\beta}(s\to t;\,\widetilde H^{\beta}_{t\leftarrow s})
:\quad
H^{\beta}_{t\leftarrow s}\leftarrow \widetilde H^{\beta}_{t\leftarrow s},
\end{equation*}
leaving all other components of the POSCM unchanged.
Similarly, a \emph{$V$-edge intervention} on dyad $j\to i$ replaces the value message function,
\begin{equation*}
    \doop_{V}(j\to i;\,\widetilde H^{V}_{i\leftarrow j})
:\quad
H^{V}_{i\leftarrow j}\leftarrow \widetilde H^{V}_{i\leftarrow j}.
\end{equation*}
If $A_{st}=0$ (resp.\ $A_{ji}=0$) in a realized world, then the corresponding message is identically zero by convention (Def.~\ref{def:msg-poscm}), so a message intervention is a no-op.
\end{definition}
Appendix Table~\ref{tab:intervention-examples} summarizes scientific readings of the primitive interventions.

\paragraph{Relation to edge/path/soft interventions in SCMs.}
In standard SCMs, ``edge'' or ``path'' interventions can be viewed as modifying a single parent contribution while leaving the rest of the mechanism fixed \citep{shpitser2016identification}.
Message-augmented POSCMs make this idea more ``surgical'': an edge-message intervention replaces a single dyadic channel $H_{i\leftarrow j}$ while leaving all other mechanisms unchanged (Def.~\ref{def:edge-message-interventions}).
This instantiates the usual notion of mechanism-level (soft) interventions, but in a setting where structure itself is stochastic and intervention-sensitive.

\begin{figure}[t]
    \centering
    \includegraphics[width=1.0\linewidth]{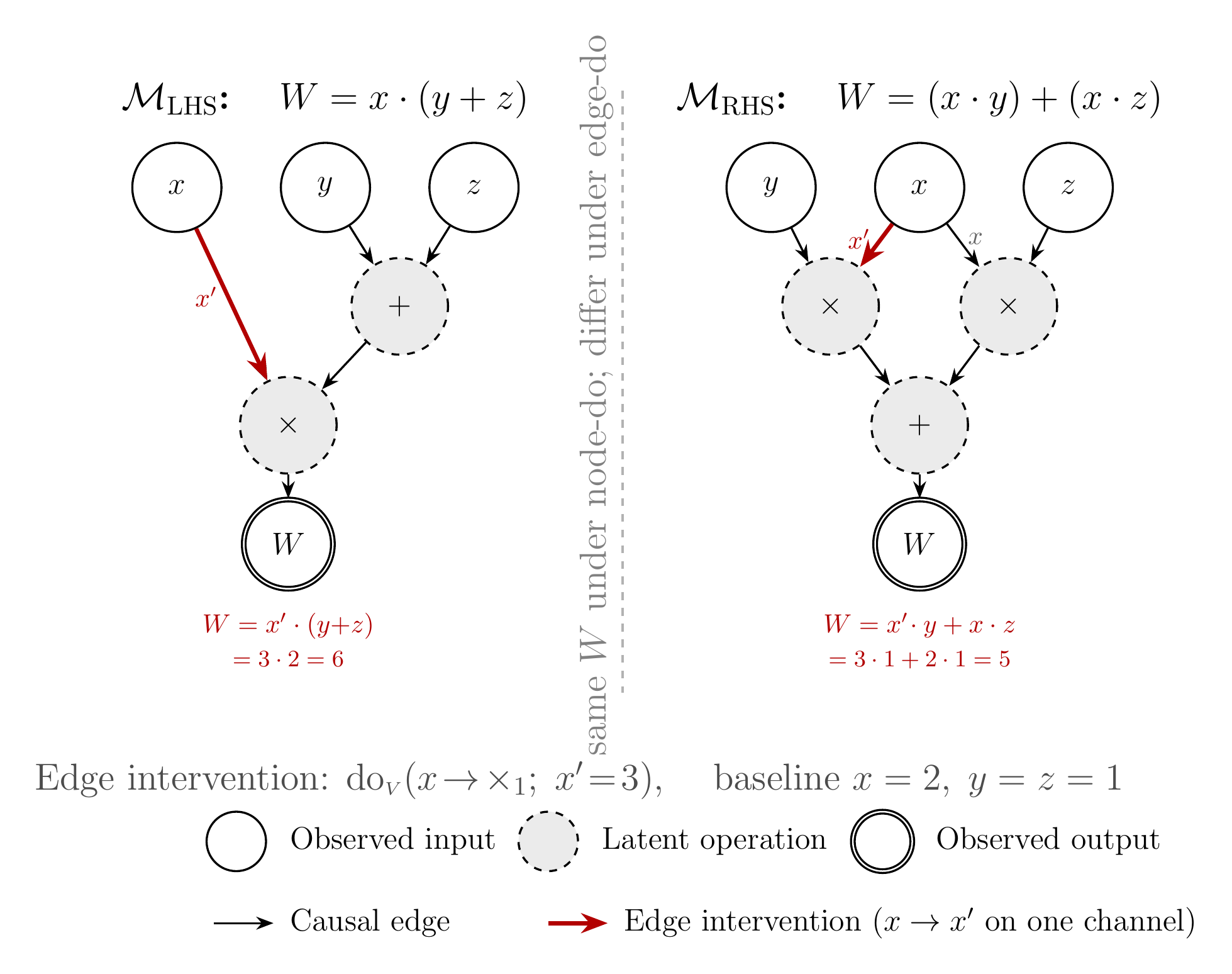}
    \caption{Distributive law toy example.}
 \label{fig:distributive_toy}
\end{figure}

\paragraph{Toy example (distributive law): why edge interventions form a strictly stronger tier (see Fig.~\ref{fig:distributive_toy}).}
Consider two deterministic causal models with inputs $(x,y,z)$ and a single observed output $W$.
Model $\mathcal{M}_{\mathrm{LHS}}$ computes $W=x\cdot(y+z)$ via an internal addition node and a single multiplication node, while model $\mathcal{M}_{\mathrm{RHS}}$ computes $W=(x\cdot y)+(x\cdot z)$ via two internal multiplications.
Suppose the intermediate arithmetic nodes are unobserved and cannot be intervened upon, so the observation model reveals only $W$.
Then for every node intervention $\doop(x=a,y=b,z=c)$ the two models agree on the observed outcome $W$, hence they are $(\mathcal{I}_{\text{node}},\Ocal)$-equivalent for this restricted interface: node interventions identify the input-output kernel, but not the internal decomposition.
In contrast, an edge intervention that perturbs only \emph{one} channel of $x$ distinguishes them.
For instance, replace only the input carried on the edge $x\to \ast_1$ by a value $x'$ while leaving the other occurrence of $x$ unchanged; then $\mathcal{M}_{\mathrm{LHS}}$ yields $W=x'(y+z)$ whereas $\mathcal{M}_{\mathrm{RHS}}$ yields $W=x'y+xz$, which differ for generic assignments (e.g., $x=2,x'=3,y=z=1$ gives $6$ vs.\ $5$).
This illustrates why edge interventions are a genuinely stronger primitive: they isolate one causal channel without globally changing the source node.
In POSCMs, the analogous separation underlies the gap between kernel identification (Theorem~\ref{thm:complete-id}) and identifiability of dyadic message primitives or edge-intervention counterfactuals under partial observability (Lemma~\ref{lem:struct_mech_confounding}, Theorem~\ref{thm:msg-id}).

\subsection{Structural endogeneity}

\begin{definition}[Per-source supervising measure]
\label{def:supervising-measure}
For a regime $r$ (baseline $0$ or an intervention regime  $\mathrm{int}$), let $\P^r$ denote the induced measure. For a source $j$, define regime-indexed outgoing-edge measure as the conditional distribution of its outgoing edge vector,
\begin{align*}
\mu_j^r(h) \ &:=\ \P^r\!\left(A_{j,>j} \mid \cH_j^r=h\right), \\ A_{j,>j}&:=(A_{j,j+1},\ldots,A_{j,N}),
\end{align*}
where $\cH_j^r$ collects the variables available to supervise outgoing edge generation from $j$ in regime $r$, including at least $\beta_j^r$. When the only supervisor being conditioned on is the source context, we write $\mu_j^r(b):=\P^r(A_{j,>j}\mid \beta_j^r=b)$.
\end{definition}

The ordered generative process means the causal graph $A$ is not fixed, but co-evolves with the context $\beta$ via $\P(A|\beta)$.

\begin{definition}[Intervention-Induced Structural Change; IISC]
\label{def:iisc}
IISC occurs for an intervention regime $\mathrm{int}$ if there exists a source $j$ such that $\mu_j^{\mathrm{int}}\neq\mu_j^0$ on the relevant conditioning support.
\end{definition}

$\mu_j^r$ factorizes when outgoing edges are conditionally independent Bernoullis given $\cH_j^r$. Although \emph{dependence across outgoing edges} yields non-product $\mu_j^r$, our IISC definitions apply to both.

\begin{remark}
Extending POSCMs beyond ordered generation is an important direction. One option is
\emph{time-unrolling} (dynamic POSCMs) in which feedback is represented via lagged variables,
yielding an acyclic graph on a finite horizon. Another option is \emph{cyclic/equilibrium}
semantics via fixed-point SCMs or $\sigma$-separation
\citep{bongers2021foundations, forre2018constraint}, which would require revisiting the
intervention semantics and IISC diagnostics in a cyclic setting.
\end{remark}

\section{Identifiability Theory}
\label{sec:identifiability}

POSCMs introduce multiple interacting mechanisms, so identifiability must be indexed by the available intervention primitives. We formalize this and establish both negative and positive results.

\begin{definition}[$(\mathcal{I},\Ocal)$-equivalence and identifiability]
\label{def:I-equiv}
Fix a measurement model $\Ocal$ (Def.~\ref{def:measurement-model}) and ordering $\tau$.
An \emph{intervention family} $\mathcal{I}$ is a collection of interventions from the primitive types in Sec.~\ref{sec:interventions} (node, edge-message), possibly restricted to a class (e.g., $\{\doop(\beta_j=b):b\in\mathcal{B}_j\}$).
Two POSCMs $M,M'$ are \emph{$(\mathcal{I},\Ocal)$-equivalent} if, for every intervention $\iota\in\mathcal{I}$, they induce identical distributions over the \emph{same} observables $\tilde Z^\iota := \Ocal(A^\iota,\beta^\iota,V^\iota)$ generated by the fixed measurement model $\Ocal$.
A parameter $\theta$ is \emph{$\mathcal{I}$-identifiable (under $\Ocal$)} if $(\mathcal{I},\Ocal)$-equivalence implies $\theta(M)=\theta(M')$, up to stated equivalence transformations.
\end{definition}

\paragraph{No ``observer swapping.''}
Equivalence is defined relative to a fixed $\Ocal$: we do \emph{not} allow changing the observation process to restore equivalence between mismatched causal models.
The following equivalence transformations are unavoidable when contexts are latent and interventions are not anchored to an intrinsic coordinate system:
(i) \emph{context reparameterization} (diffeomorphisms $\gamma:\mathcal{B}\to\mathcal{B}$ with appropriately transformed kernels and context mechanisms), and
(ii) \emph{discrete label permutation} when $\beta_i$ takes values in a finite set.
In addition, if message coordinates are obtained via a non-unique functional representation (e.g., KAS), then identification of \emph{message coordinates} is understood either relative to a fixed representation or ``up to'' the induced representation gauge (Remark~\ref{rem:kas-gauge}).

\paragraph{Intervention-family notation.}
We use $\mathcal{I}_{\mathrm{obs}}$ for the baseline observational regime, $\mathcal{I}_{\beta\text{-node}}$ for admissible interventions $\{\doop(\beta_j=b)\}$, $\mathcal{I}_{V\text{-node}}$ for admissible interventions $\{\doop(V_j=v)\}$, $\mathcal{I}_{\beta\text{-edge}}$ for admissible context-message interventions $\{\doop_\beta(s\to t;\widetilde H^\beta_{t\leftarrow s})\}$, and $\mathcal{I}_{V\text{-edge}}$ for admissible value-message interventions $\{\doop_V(j\to i;\widetilde H^V_{i\leftarrow j})\}$. When joint interventions are required, this is stated explicitly in the corresponding assumption.
\begin{proposition}[Non-identifiability without $\beta$-interventions]
\label{prop:nonid}
Let $\mathcal{I} = \mathcal{I}_{\mathrm{obs}} \cup \mathcal{I}_{V\text{-node}} \cup \mathcal{I}_{V\text{-edge}}$ (no $\beta$-level interventions). If $\beta$ is unobserved, then:
\begin{enumerate}[label=(\roman*),nosep,leftmargin=*]
    \item The structure kernel $\P(A\mid \beta)$ (equivalently, the baseline supervising measures $\mu_j^0(\beta_j) := \P^0(A_{j,>j}\mid \beta_j)$) is not $\mathcal{I}$-identifiable beyond context reparameterization.
    \item The context propagation mechanisms $\{\phi_i\}$ (equivalently, the conditional distributions $\P(\beta_i \mid \beta_{\Pa_A(i)},\Pa_A(i))$) are not $\mathcal{I}$-identifiable beyond context reparameterization.
    \item The mechanism-assignment operator $\Gamma$ (equivalently, the conditional distributions $\P(f_i \in \Fcal \mid \beta_i,\Pa_A(i))$, and hence the induced value kernels $\P(V_i \mid V_{\Pa_A(i)},\beta_i,\Pa_A(i))$) is not $\mathcal{I}$-identifiable beyond context reparameterization.
\end{enumerate}
\end{proposition}


Full proof in Appendix~\ref{app:proof:prop:nonid}.

Proposition~\ref{prop:nonid} establishes that $V$-level data alone cannot disentangle the structure kernel, context propagation, or mechanism assignment when contexts are latent. Thus the barrier in Proposition~\ref{prop:nonid} is that the observational outgoing-edge measure $\mu_j^0$ underdetermines the interventional measure $\mu_j^{\beta\text{-do}}$; context interventions close this gap on the reachable support in Theorem~\ref{thm:complete-id}. A second barrier arises when edges are also latent.

\begin{lemma}[Structure-mechanism confounding under latent edges]
\label{lem:struct_mech_confounding}
When edges are latent (not observed via $\mathcal{O}$), node-level interventions $\mathrm{do}(V_j = \tilde{v})$ are insufficient to determine the distribution resulting from an $V$-edge intervention, even with degenerate (non-random) contexts.
\end{lemma}


Full proof in Appendix~\ref{app:proof:lem:struct_mech_confounding}.

Together, Proposition~\ref{prop:nonid} and Lemma~\ref{lem:struct_mech_confounding} identify two distinct barriers to identification: (1) latent contexts create reparameterization symmetries, and (2) latent edges create structure-mechanism confounding. Positive identification requires addressing both.

\paragraph{Assumptions for the positive result.}
We assume access to (i) probing-based structure readout for latent adjacency, (ii) context and value readout, (iii) joint $\beta$-node and $V$-node intervention coverage, and (iv) positivity.
Formal statements are in Appendix~\ref{app:kernel-assumptions}.

\begin{theorem}[Kernel identifiability]
\label{thm:complete-id}
Let $\mathcal{I}_{\mathrm{kern}} := \mathcal{I}_{\mathrm{obs}} \cup \mathcal{I}_{\beta\text{-node}} \cup \mathcal{I}_{V\text{-node}}$ and fix the measurement model $\Ocal$ and ordering $\tau$ (Def.~\ref{def:I-equiv}). Assume formal assumptions in Appendix~\ref{ass:struct-readout}--\ref{ass:positivity}. Then, on the intervention-reachable support, the POSCM kernels $\alpha$, $\{\phi_i\}$, and $\Gamma$ are $(\mathcal{I}_{\mathrm{kern}},\Ocal)$-identifiable in the following kernel sense:
\begin{enumerate}[label=(\roman*),nosep,leftmargin=*,labelsep=0.3em, labelwidth=0pt, align=left]
    \item \textbf{Structure kernel $\alpha$.} For each node $j$, the interventional supervising measure
    \[
        \mu_j^{\beta\text{-do}}(b) := \P\!\bigl(A_{j,>j}\mid \doop(\beta_j=b)\bigr)
    \]
    is identifiable. In particular, each dyadwise marginal $b\mapsto \P(A_{ji}=1\mid \doop(\beta_j=b))$ for $i>j$ is identifiable.
    \item \textbf{Context propagation kernels induced by $\{\phi_i\}$.} For each node $i$ and each realized parent set $S$ in the support of $\Pa_A(i)$ under $\mathcal{I}_{\mathrm{kern}}$, the conditional context kernel
    \[
        K^{\beta}_{i,S}(\cdot\mid b_S)
        \;:=\;
        \P\!\bigl(\beta_i \mid \doop(\beta_S=b_S),\ \Pa_A(i)=S\bigr)
    \]
    is identifiable as a function of $b_S$ on the intervention-reachable support. Equivalently, the conditional distribution $\P(\beta_i\mid \beta_{\Pa_A(i)},\Pa_A(i))$ is identifiable (up to the equivalences discussed above when interventions are not anchored).
    \item \textbf{Endogenous-variable mechanism kernels induced by $\Gamma$.} For each node $i$ and each realized parent set $S$ in the support of $\Pa_A(i)$ under $\mathcal{I}_{\mathrm{kern}}$, the conditional value kernel
    \begin{align*}
        &K^{V}_{i,S}(\cdot\mid v_S,b_i)
        :=\\
        &\P\!\bigl(V_i \mid \doop(V_S=v_S),\ \beta_i=b_i,\ \Pa_A(i)=S\bigr)
    \end{align*}
    is identifiable as a function of $(v_S,b_i)$ on the intervention-reachable support. Equivalently, the conditional distribution $\P(V_i\mid V_{\Pa_A(i)},\beta_i,\Pa_A(i))$ is identifiable.
\end{enumerate}
Identification is purely distributional and uses repeated sampling under interventions; adjacency is accessed through the probing-based readout in Assumption~\ref{ass:struct-readout}.
\end{theorem}

Full proof in Appendix~\ref{app:proof:thm:complete-id}
\paragraph{Estimator implied by Theorem~\ref{thm:complete-id}.}
The theorem suggests the following experimental pipeline: 1) fix or estimate an ordering $\tau$; 2) calibrate context and value readouts; 3) use repeated Phase-II probes to recover $A$ or a posterior over parent sets; 4) estimate $\mu_j^{\beta\text{-do}}(b)=\P(A_{j,>j}\mid\doop(\beta_j=b))$ under $\beta$-node interventions; 5) stratify by identified parent set and context to estimate $K_i^\beta$ and $K_i^V$ under joint $\beta$- and $V$-node interventions. The result is therefore an experimental-design statement, not a claim that these kernels are recoverable from passive $V$-data alone.

Appendix Table~\ref{tab:assumption-relaxations} summarizes how the sufficient assumptions relax, and Appendix~\ref{app:kernel-assumptions} gives the finite-strata sample-scaling guide for plug-in kernel estimation on reachable support. Appendix~\ref{app:intervention-hierarchy} contains further discussion of intervention minimality and dyadic message identifiability (including Theorem~\ref{thm:msg-id} and Proposition~\ref{prop:instance-id}).

\section{Experiments}
\label{sec:experiments}

We validate our identifiability results (Proposition~\ref{prop:nonid}, Lemma~\ref{lem:struct_mech_confounding}, Theorem~\ref{thm:complete-id}) in a biophysically detailed virtual human retina simulator (NEURON; ModelDB \#2018247; \citep{LY2025144}). The retina is a natural POSCM: cell type co-determines synaptic connectivity and synaptic transfer functions, and the laminar circuit supports a biologically grounded ordered-generation semantics. Tables~\ref{tab:retina-dictionary-main} and~\ref{tab:retina-theory-main} give the POSCM--retina dictionary and the theory--experiment correspondence used throughout this section.

\begin{table}[t]
\centering
\begin{tabular}{p{0.30\linewidth}p{0.60\linewidth}}
\toprule
POSCM object & Retina instantiation\\
\midrule
$\tau$ & Layer order PR $\to$ HZ $\to$ BC $\to$ AC $\to$ RGC. \\
$\beta_i$ & Cell type, supervises both wiring and synaptic physiology. \\
$A_{ji}$, $\alpha$ & Synapse indicator and type- and distance-dependent connectivity kernel. \\
$\Gamma/f_i$ & Synaptic rule and parameters, e.g., sign, gain, threshold, slope, and time constant. \\
$V_i$ & Membrane-potential activity over the simulated epoch. \\
\bottomrule
\end{tabular}
\caption{POSCM--retina dictionary used in the virtual retina experiments}
\label{tab:retina-dictionary-main}
\end{table}

\begin{table}[t]
\centering
\begin{tabular}{p{0.30\linewidth}p{0.60\linewidth}}
\toprule
Theory target & Simulation role \\
\midrule
Prop.~\ref{prop:nonid} / Exp.~1 & Non-identifiability without $\beta$ access: type (i.e., mechanism) swaps remain indistinguishable under the restricted $V$-level interface. \\
Lemma~\ref{lem:struct_mech_confounding} / Exp.~2 & Latent-edge structure--mechanism confounding: calibrated $(p,g)$ pairs match under node clamps but diverge under edge conductance tests. \\
Thm.~\ref{thm:complete-id} / Exp.~3 & Kernel recovery with readout and interventions: context sweeps modulate composition and voltage sweeps recover a pathway transfer curve. \\
\bottomrule
\end{tabular}
\caption{Theory--experiment correspondence. Exp.~1 targets Proposition~\ref{prop:nonid}, Exp.~2 targets Lemma~\ref{lem:struct_mech_confounding}, and Exp.~3 targets Theorem~\ref{thm:complete-id}.}
\label{tab:retina-theory-main}
\end{table}

Crucially, real retinal experiments exhibit precisely the partial observability that motivates POSCMs. Electrophysiological recordings (e.g., patch clamp or extracellular arrays) reveal neural activity but not cell type \citep{MEISTER199495}; conversely, anatomical reconstructions such as serial-section electron microscopy recover wiring and morphological identity but not activity \citep{Briggman2011}. Most standard preparations offer no perturbation capability \citep{MEISTER199495}. When perturbations are available, they span multiple tiers of our intervention hierarchy: axon transection or laser ablation severs individual connections ($V$-edge interventions) \citep{10.3389/fncel.2019.00423}, genetic inactivation under cell-type-specific promoters silences entire cell classes ($\beta$-node interventions) \citep{Montgomery2010}, and activity-dependent manipulations during development can rewire connectivity itself (structure-kernel interventions) \citep{Shen2020}. A virtual retina simulator lets us access all of these intervention types with full ground-truth knowledge, enabling controlled tests of each identifiability prediction.

A small retinal patch contains a functionally complete microcircuit; we simulate independent random patches (``seeds'') as independent POSCM instances. Chemical synapses implement graded release often modeled with an approximately $\tanh$ activation:

\begin{equation}
s_\infty = \tanh\!\left(\frac{V_{\mathrm{pre}} - V_{\mathrm{thr}}}{V_{\mathrm{slope}}}\right), \
I_{\mathrm{syn}} = g_{\max}\, s \,(V_{\mathrm{post}} - e_{\mathrm{rev}}),
\label{eq:retina-synapse}
\end{equation}

with first-order binding kinetics $\dot{s} = (s_\infty - s)/[\tau(1-s_\infty)]$. Typical parameters are $V_{\mathrm{thr}} \approx -45~\text{mV}$, $V_{\mathrm{slope}} = 10~\text{mV}$, $\tau = 10~\text{ms}$, and $g_{\max}=0.00256~\mu\text{mho}$.

\subsection{Experiment 1: Non-identifiability without $\beta$-interventions}
\label{sec:retina-exp1}

Proposition~\ref{prop:nonid} predicts that if cell types $\beta$ are unobserved and no $\beta$-level interventions are available, then $(\alpha,\{\phi_i\},\Gamma)$ are not identifiable beyond context reparameterization. We instantiate this by constructing a \emph{type-swapped} twin retina model that is observationally indistinguishable even under $V$-node interventions.

\paragraph{Setup (twin models).}
Model $M$ uses standard parameters. ON bipolar cells (ON-BC) receive sign-inverting mGluR synapses with conductance $g_{\mathrm{ON}}=+0.00256~\mu\text{mho}$ and threshold $V_{\mathrm{thr}}^{\mathrm{ON}}=-40$\,mV; OFF-BCs receive sign-preserving iGluR synapses with $g_{\mathrm{OFF}}=-0.00256~\mu\text{mho}$ and $V_{\mathrm{thr}}^{\mathrm{OFF}}=-42$\,mV. Model $M'$ swaps ON-BC and OFF-BC labels (and swaps the corresponding synaptic parameters). In both models, $\beta$ is \emph{not recorded}.

\paragraph{Protocol.}
For each of $B=2$ seeds, we run 10 worlds: observational ($M$ and $M'$), plus $V$-node interventions that voltage-clamp a single photoreceptor ($\mathrm{PR}_{\text{upper}}[0]$) at $v \in \{-60,-50,-40,-30\}$\,mV in each model (NEURON \texttt{SEClamp}, series resistance $0.001~\Omega$).

\paragraph{Analysis.}
From each cell trace we extract a firing rate via threshold crossings at $-20$\,mV and compare the across-cell firing-rate distributions between $M$ and $M'$ with a two-sample Kolmogorov--Smirnov test:
$H_0: F_M = F_{M'}$, $D = \sup_r \left|F_M(r)-F_{M'}(r)\right|.$

\begin{table}[t]
\centering
\begin{tabular}{lccccc}
\toprule
& Obs. & $-60$ & $-50$ & $-40$ & $-30$ \\
\midrule
KS $D$ & 0.0225 & 0.0294 & 0.0294 & 0.0277 & 0.0294 \\
$p$-value & 0.999 & 0.964 & 0.964 & 0.980 & 0.964 \\
\bottomrule
\end{tabular}
\caption{Experiment~1: KS test comparing firing-rate distributions between $M$ and the type-swapped twin $M'$. Columns $-60$ to $-30$ denote $\doop(V_{\mathrm{PR}}=v)$ in mV. The high $p$-values ($\gg 0.05$) are consistent with $(\mathcal{I},\Ocal)$-equivalence
under the restricted $V$-level interface, as predicted by Proposition~\ref{prop:nonid}.}
\label{tab:retina-exp1-ks}
\end{table}

\paragraph{Results.} Table~\ref{tab:retina-exp1-ks} shows that swapping ON/OFF type labels together with the corresponding synaptic rules preserves the induced $V$-distributions when types are latent, empirically validating the reparameterization symmetry underlying Proposition~\ref{prop:nonid}. This type swap is not itself a POSCM intervention $\doop(\beta=b)$; it constructs an alternative POSCM $M'$ with a different context and mechanism labeling but the same observable measure under the available $V$-level interface. Thus voltage-clamp data without $\beta$ access cannot break the type/mechanism symmetry.

\subsection{Experiment 2: Structure--mechanism confounding under latent edges}
\label{sec:retina-exp2}

Lemma~\ref{lem:struct_mech_confounding} shows that when edges are latent, $V$-node interventions do not suffice to predict $V$-edge intervention effects. We construct two retina models with different PR$\to$BC connectivity densities whose $V$-responses match under node-level clamps but diverge under edge-level conductance perturbations.

\paragraph{Setup (calibrated confounding pair).}
Model $M$ uses the natural PR$\to$BC synapse density $p$ and baseline conductance $g=0.00256~\mu\text{mho}$. Model $M'$ randomly blocks 40\% of PR$\to$BC synapses (sets $g_{\max}=0$), yielding $p' = 0.6p$, and scales the remaining conductances to $g' = g/0.6 = 0.00427~\mu\text{mho}$ so that $p\,g = p'\,g'$ (matching mean synaptic drive in the linear regime).

\paragraph{Protocol.}
For each of $B=2$ seeds we run (i) \emph{node-do} conditions that clamp a single PR at $v\in\{-60,-50,-40,-30\}$\,mV, and (ii) \emph{edge-do} conditions that replace the conductance on all active PR$\to$BC synapses with a test value $g_{\text{test}} \in \{0.001,0.002,0.004,0.008\}~\mu\text{mho}$ (preserving blocked synapses in $M'$).

\paragraph{Analysis.}
For each cell we compute a steady-state effect $\Delta V_i = \bar V_i^{\text{(int)}} - \bar V_i^{\text{(obs)}}$, where $\bar V_i$ is the mean potential over the last 50\% of the trace. We focus on the postsynaptic BC population (BIP$_{\text{upper}}$) and compare effect distributions between $M$ and $M'$ using maximum mean discrepancy (MMD) with an RBF kernel: $\mathrm{MMD}^2(X,Y) = \mathbb{E}[k(X,X')] - 2\,\mathbb{E}[k(X,Y)] + \mathbb{E}[k(Y,Y')]$, $k(x,y)=\exp\!\left(-\frac{\|x-y\|^2}{2\sigma^2}\right)$.

\begin{table}[t]
\centering
\begin{tabular}{lclc}
\toprule
\multicolumn{2}{c}{Node-do} & \multicolumn{2}{c}{Edge-do} \\
Condition & MMD & Condition & MMD \\
\midrule
 $v=-60$\,mV & 0.0440 & $g_{\text{test}}=0.001$ & 0.1135 \\
 $v=-50$\,mV & 0.0435 & $g_{\text{test}}=0.002$ & 0.0950 \\
 $v=-40$\,mV & 0.0465 & $g_{\text{test}}=0.004$ & 0.1853 \\
 $v=-30$\,mV & 0.0437 & $g_{\text{test}}=0.008$ & 0.1309 \\
 Mean & \textbf{0.0444} & Mean & \textbf{0.1312} \\
\bottomrule
\end{tabular}
\caption{Experiment~2: MMD for the calibrated pair $(M,M')$. Node-do effects remain similar, while edge-do tests reveal the latent structural difference.}
\label{tab:retina-exp2-mmd}
\end{table}

\paragraph{Results.} Table~\ref{tab:retina-exp2-mmd} shows that node-level clamps keep the calibrated pair hard to distinguish, whereas edge-level interventions expose the differing synapse counts, consistent with Lemma~\ref{lem:struct_mech_confounding}. The calibration $p\,g=p'\,g'$ is exact only for a linearized synapse; under the nonlinear $\tanh$ dynamics in Eq.~\eqref{eq:retina-synapse}, the node-do MMD remains small but nonzero. The edge-do mean MMD ($0.1312$) is about $3\times$ the node-do mean MMD ($0.0444$); finite sample detectability can be assessed with a permutation or bootstrap null.

\subsection{Experiment 3: Kernel identifiability via $\beta$-node and $V$-node interventions}
\label{sec:retina-exp3}

Theorem~\ref{thm:complete-id} states that with $\beta$-node and $V$-node interventions (plus readout/coverage assumptions) the POSCM kernels become identifiable. We demonstrate two corresponding identification routes in the retina: a context sweep (proxy $\beta$-node intervention) that changes cell-type composition, and a voltage sweep ($V$-node interventions) that recovers a pathway-level transfer curve.

\paragraph{Protocol.} During phase I (context sweep),
we vary the eccentricity parameter (a global context proxy) over $\{-1.2,-1.5,-2.0,-2.5,-3.0,-3.5\}$\,mm before network construction, producing different type compositions and densities. During phase II (voltage sweep), we simultaneously clamp all bipolar cells (BIP$_{\text{upper}}$) to $v \in \{-70,-60,-50,-40,-30,-20\}$\,mV (bulk \texttt{SEClamp}) and measure the induced effect on ganglion cells.

\begin{table}[t]
\centering
\begin{tabular}{lcccccc}
\toprule
Ecc. & $-1.2$ & $-1.5$ & $-2.0$ & $-2.5$ & $-3.0$ & $-3.5$ \\
\midrule
Cells & 289 & 306 & 330 & 347 & 350 & 346 \\
\bottomrule
\end{tabular}
\caption{Experiment~3 (Phase I): eccentricity sweep changes network composition; $-1.2$ mm is baseline.}
\label{tab:retina-exp3-ecc}
\end{table}

\paragraph{Results} Table~\ref{tab:retina-exp3-ecc} shows that changing eccentricity changes network composition, providing context-level variation. For the voltage sweep,
we compute $\Delta \bar V_{\mathrm{RGC}} = \bar V_{\mathrm{RGC}}^{\doop(V_{\mathrm{BC}}=v)} - \bar V_{\mathrm{RGC}}^{\text{(obs)}}$ (Fig.~\ref{fig:retina-transfer-curve}; Appendix Table~\ref{tab:retina-exp3-transfer}). The sweep traces a sigmoidal BC$\to$RGC pathway response; thhe affine-$\tanh$ guide has midpoint near $-37$\,mV and width about $5.3$\,mV, so we interpret it as pathway-level recovery rather than direct recovery of single-synapse parameters. Together, the context and voltage sweeps illustrate the readout- and intervention-driven kernel recovery in Theorem~\ref{thm:complete-id}.

\begin{figure}[t]
\centering
\includegraphics[width=\linewidth]{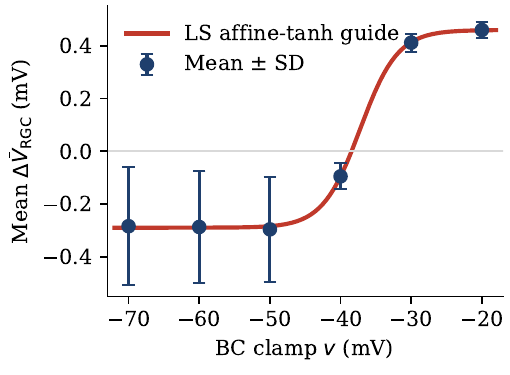}
\caption{BC$\to$RGC voltage sweep. Points show the mean $\Delta\bar V_{\mathrm{RGC}}\pm$ SD (Appendix Table~\ref{tab:retina-exp3-transfer}); The affine-$\tanh$ curve is a pathway-level visual guide, not direct recovery of the single-synapse nonlinearity in Eq.~\eqref{eq:retina-synapse}.}
\label{fig:retina-transfer-curve}
\end{figure}

\paragraph{Gene-regulatory analogue.}
As a second-domain check, we repeated the experiments in SERGIO, an external gene-regulatory network (GRN) simulator \citep{dibaeinia2020sergio} (Appendix~\ref{app:grn-analogue}).

\section{Discussion}
POSCMs extend SCMs to settings where interaction structure is generated by upstream context and can itself be intervened on, while retaining acyclic ordered-generation semantics. The resulting intervention vocabulary spans context, value nodes and edge-local channels, and the identifiability results diagnose both positive conditions and failure modes: latent-context symmetries and latent-edge structure–mechanism confounding. The retina experiments and SERGIO-backed GRN analogue (Appendix~\ref{app:grn-analogue}) confirm these analytical results in external simulators under the proposed readout and coverage assumptions. Future work should weaken access assumptions, handle cyclic or time-unrolled settings, and separate dyadic message identification questions from kernel identification ones.



\begin{acknowledgements} 
    Turan Orujlu is funded by the Deutsche Forschungsgemeinschaft (DFG, German Research Foundation) under Germany’s Excellence Strategy – EXC number 2064/1 – Project number 390727645. Charley M. Wu is supported by the Deutsche
    Forschungsgemeinschaft (German Research Foundation, DFG) under
    Germany’s Excellence Strategy (EXC 3066/1 ``The Adaptive Mind'',
    Project No. 533717223), and the Excellence Cluster ``Reasonable AI'' by the Deutsche Forschungsgemeinschaft (German Research Foundation, DFG) under Germany’s Excellence Strategy – EXC-3057. The authors thank the International Max Planck Research School for Intelligent Systems (IMPRS-IS) for supporting Turan Orujlu. The authors also thank Eric J. Tchetgen Tchetgen, Ilya Shpitser, Sacha Sokoloski, Trung Phung, Helen Guo,  Jacob Chen, and Beatrix Wen for insightful discussions.
\end{acknowledgements}

\bibliography{uai2026-template}

\onecolumn

\title{Partially Observed Structural Causal Models\\(Supplementary Material)}
\maketitle

\appendix
\section{Related Work}
\label{sec:related_work}

\paragraph{Causal discovery over fixed graphs.}
Classic causal discovery methods, constraint-based (e.g., Peter-Clark, fast causal inference) and score-based (e.g., greedy equivalence search), typically aim to recover a \emph{fixed} causal graph from samples, with extensions handling latent confounding, interventions, and partial observability \citep{spirtes2000causation,chickering2002optimal,silva2006latent}.
These methods assume the causal graph is a fixed object to be recovered. This precludes modeling settings where the interaction structure is itself a stochastic outcome of upstream processes where interventions may alter not just the mechanisms on a given graph, but the measure that generates the graph. POSCMs address this limitation by treating structure formation as an endogenous causal mechanism subject to its own intervention semantics (Sec.~\ref{sec:interventions}).

\paragraph{Context-dependent mechanisms.}
A central idea in modern causal inference is that mechanisms may shift across environments, and that such shifts aid identification. Invariance-based frameworks exploit this: invariant causal prediction \citep{Peters2016} identifies causal parents as those whose conditional holds invariant across environments, joint causal inference \citep{mooij2020joint} pools observational and interventional datasets by modeling regime explicitly, and transportability theory \citep{pearl2014external} characterizes when causal queries transfer across domains. A related line of work makes the regime indicator a first-class variable in the graphical model: \citet{dawid2020decision}'s decision-theoretic framework introduces a regime node to unify observational and interventional measures, while the data-fusion literature uses selection variables to flag which node-level mechanisms differ between settings \citep{bareinboim2013meta, Perry2022, lee2024fusion}. Separately, context-specific independence models and their graphical generalizations (staged trees, chain event graphs, context-specific trees) represent conditional independencies that hold only in certain regimes \citep{Boutilier1996,collazo2020ceg,Riccomagno2007,Duarte2024}. Across all these frameworks, context or regime modulates \emph{node-level} mechanisms, i.e., the full conditional $P(X_i|\Pa_i)$. They do not decompose mechanisms into edge-specific components that may vary independently with context. This matters in practice: for instance, genetically encoded tools can now ablate individual synapses in neural circuits \citep{10.7554/eLife.103757}, an intervention that targets a single edge mechanism rather than an entire node. POSCMs extend the context-dependence picture to this finer granularity: contexts are upstream regime variables that jointly supervise both the node-level value mechanisms and the edge-formation measure, so that a shift in context can alter not only how a variable responds to its parents, but which parents it has and how each parent's influence is transmitted.

\paragraph{LDAGs and context-specific causal inference.}
Labelled DAGs (LDAGs) encode context-specific independencies by attaching labels to edges: in a context where a label is active, the corresponding local parent contribution is inactive \citep{tikka2019context,mokhtarian2022context}. POSCMs recover this deterministic observed-context case by taking $A_{ji}=g_{ji}(\beta)$. If $g_{ji}(b)=0$, the $j\to i$ channel contribution is zero and the local kernel for $V_i$ is invariant to $V_j$ given the active parent set. This is a local mechanism restriction, not a claim of marginal independence. POSCMs extend the LDAG picture to stochastic edge formation, latent or noisy contexts, and interventions on the graph-generation measure itself.

\paragraph{Latent contexts and partial observability.}
Because contexts can be latent, POSCMs connect causal discovery with latent variables and measurement error \citep{carroll2006measurement,schennach2016measurement}, and causal representation learning, which seeks latent factors that render mechanisms stable across environments \citep{scholkopf2021crl}.
Additionally, since POSCMs can model interacting entities whose latent contexts stochastically generate the interaction structure itself, they also connect to relational causal models \citep{maier2013relational} which formalize causal reasoning over repeated interacting units. However, the aforementioned models do not treat edge formation as an interventionable causal mechanism.

\paragraph{Functional decomposition and edge-level interventions.}
To intervene on individual edges, one needs a decomposition of node mechanisms into per-parent contributions. A long line of work exploits functional assumptions (e.g., additive noise) to identify causal directionality and decompose mechanisms into interpretable components \citep{buhlmann2014cam}. Existing edge-level intervention formalisms take a different approach: the hierarchy of \citet{shpitser2016identification} defines edge interventions as value-routing, i.e., sending different values along different pathways within a fixed mechanism, while soft interventions \citep{Correa2020, JaberKSB20} modify mechanisms but typically at the node level. Our message-augmented representation (Sec.~\ref{sec:message-aug}) plays an analogous role to additive-noise decompositions: it separates a context-to-structure kernel from dyad-local value mechanisms, providing the functional primitives on which true edge-mechanism interventions (Sec.~\ref{sec:interventions}) are defined.

\section{KAS edge-functional decomposition beyond continuous compact mechanisms}
\label{app:kas_density}

KAS-style decompositions are classically stated for \emph{continuous} mechanisms on \emph{compact} domains (e.g., $[0,1]^n$) \citep{kolmogorov1957,arnold1957,sprecher1996}.
In contrast, SCMs (and hence POSCMs) typically allow arbitrary measurable mechanisms on non-compact domains (e.g., $\R^n$), and may be discontinuous.
For probabilistic modeling, it is natural to measure approximation error in $L^p(\mu)$ under the parent distribution $\mu$.
The following standard truncation-approximation argument shows that KAS edge-functional decomposition remains dense in this sense.

\begin{lemma}[An $L^p(\mu)$ density statement for KAS edge-functional decompositions]
\label{lem:kas_density_lp}
Let $X\sim \mu$ be an $\mathbb{R}^n$-valued random vector and let $f:\mathbb{R}^n\to \mathbb{R}$
be a measurable function with $\mathbb{E}\!\left[|f(X)|^p\right]<\infty$ for some $1\le p<\infty$.
For any $\varepsilon>0$, there exists a measurable function $g:\mathbb{R}^n\to \mathbb{R}$ such that

\begin{equation*}
    \mathbb{E}\!\left[\,|f(X)-g(X)|^p\,\right] < \varepsilon,    
\end{equation*}

and such that $g$ is representable in KAS edge-functional decomposition form (on a sufficiently large compact cube),
in the following sense: there exists $R<\infty$ for which $g|_{[-R,R]^n}$ admits a KAS representation after affine rescaling of $[-R,R]^n$ to $[0,1]^n$.
\end{lemma}

\begin{proof}
Write $\|\cdot\|_{p}$ for the $L^p(\mu)$ norm:
\begin{equation*}
    \|h\|_{p} := \Big(\mathbb{E}\big[|h(X)|^p\big]\Big)^{1/p}.    
\end{equation*}

Fix $\varepsilon>0$ and write $\varepsilon_p := \varepsilon^{1/p}$.
It suffices to construct $g$ with $\|f-g\|_{p}<\varepsilon_p$, since then
$\mathbb{E}[|f(X)-g(X)|^p]=\|f-g\|_{p}^p<\varepsilon$.
We split this $L^p(\mu)$ error budget evenly and set $\delta := \varepsilon_p/3$.

\paragraph{Step 1 (domain truncation).}
Let $C_R := [-R,R]^n$ and define $f_R(x):= f(x)\mathbf{1}\{x\in C_R\}$. Then
\begin{equation*}
    \|f-f_R\|_{p}^p
= \mathbb{E}\big[\,|f(X)|^p\,\mathbf{1}\{\|X\|_\infty>R\}\big].
\end{equation*}
As $R\to\infty$, the indicator $\mathbf{1}\{\|X\|_\infty>R\}$ decreases pointwise to $0$, and the integrand is dominated
by $|f(X)|^p\in L^1$. By the dominated convergence theorem \cite{lebesgue1910},
$\mathbb{E}[\,|f(X)|^p\mathbf{1}\{\|X\|_\infty>R\}]\to 0$.
Hence we may choose $R$ such that
\begin{equation*}
    \|f-f_R\|_{p} < \delta.
\end{equation*}

\paragraph{Step 2 (range truncation / clipping).}
For $B>0$, define the clipping operator
\begin{equation*}
    \operatorname{clip}_B(y) := \operatorname{sign}(y)\min\{|y|,B\}.    
\end{equation*}
This map is $1$-Lipschitz on $\mathbb{R}$: for all $a,b\in\mathbb{R}$,
\begin{equation*}
    \bigl|\operatorname{clip}_B(a)-\operatorname{clip}_B(b)\bigr| \le |a-b|.
\end{equation*}
In particular, if $|b|\le B$ (so $\operatorname{clip}_B(b)=b$), then
\begin{equation*}
|a-\operatorname{clip}_B(a)|
\le |a-b| + \bigl|\operatorname{clip}_B(a)-\operatorname{clip}_B(b)\bigr|
\le 2|a-b|.
\end{equation*}

Let $\tilde f(x) := \operatorname{clip}_B(f_R(x))$. Then $\tilde f$ is measurable, supported on $C_R$, and $|\tilde f|\le B$.
Moreover, for each $x$,
\begin{equation*}
    |f_R(x)-\tilde f(x)|
    = \bigl(|f_R(x)|-B\bigr)_+
    = \bigl(|f_R(x)|-B\bigr)\mathbf{1}\{|f_R(x)|>B\},
\qquad\text{so}\qquad
|f_R(x)-\tilde f(x)|^p \le |f_R(x)|^p\,\mathbf{1}\{|f_R(x)|>B\}.
\end{equation*}

Since $|f_R(X)|^p \le |f(X)|^p$ and $|f(X)|^p$ is integrable, the dominated convergence theorem \citep{lebesgue1910} gives
$\mathbb{E}[|f_R(X)|^p\mathbf{1}\{|f_R(X)|>B\}]\to 0$ as $B\to\infty$.
Hence we may choose $B$ such that
\begin{equation*}
    \|f_R-\tilde f\|_{p} < \delta.
\end{equation*}

\paragraph{Step 3 (Lusin + Tietze: continuous approximation on $C_R$).}
Set $\eta := (\delta/2B)^p$. By Lusin's theorem \citep{lusin1912}, there exists a compact $K\subseteq C_R$ such that $\tilde f|_K$ is continuous and
$\mu(C_R\setminus K)<\eta$.
By the Tietze extension theorem (since $C_R$ is normal) \citep{tietze1915}, there exists a continuous function
$h:C_R\to\mathbb{R}$ such that $h=\tilde f$ on $K$ and $\sup_{x\in C_R}|h(x)|\le B$.
Since $h=\tilde f$ on $K$ and both are bounded by $B$ on $C_R$, for $x\in C_R\setminus K$ we have
$|\tilde f(x)-h(x)|\le 2B$. Therefore,
\begin{equation*}
    \|\tilde f-h\|_{p}^p
= \mathbb{E}\!\left[|\tilde f(X)-h(X)|^p\,\mathbf{1}\{X\in C_R\}\right]
\le (2B)^p\,\mu(C_R\setminus K)
< (2B)^p\,\eta
= \delta^p,
\end{equation*}
hence
\begin{equation*}
    \|\tilde f-h\|_{p} < \delta.
\end{equation*}

\paragraph{Step 4 (KAS representation on the compact cube).}
Define the affine homeomorphism $T_R:C_R\to [0,1]^n$ by
\begin{equation*}
    T_R(x):=\frac{x+R\mathbf{1}}{2R}\quad\text{(componentwise)},\qquad
T_R^{-1}(u):=2Ru-R\mathbf{1},    
\end{equation*}

where $\mathbf{1}$ is the all-ones vector in $\mathbb{R}^n$.
Let $\bar h(u):= h(T_R^{-1}(u))$, which is continuous on $[0,1]^n$.

By the Kolmogorov-Arnold-Sprecher representation theorem for continuous functions on $[0,1]^n$ \citep{kolmogorov1957,arnold1957,sprecher1996}, there exist
real constants $\eta,\lambda_1,\dots,\lambda_n$, a continuous function $\Phi:\mathbb{R}\to\mathbb{R}$,
and a continuous increasing function $\varphi:[0,1]\to[0,1]$ such that for all $u\in[0,1]^n$,
\begin{equation*}
    \bar h(u)
=
\sum_{q=0}^{2n}
\Phi\!\left(
\sum_{p=1}^{n} \lambda_p\,\varphi(u_p+\eta q) + q
\right).
\end{equation*}
Define $k(u)$ to be the right-hand side, and set
\begin{equation*}
    g_R(x) := k(T_R(x))\quad\text{for }x\in C_R.
\end{equation*}
Then \emph{exactly} $g_R(x)=h(x)$ for all $x\in C_R$.
Finally, define a measurable $g:\mathbb{R}^n\to\mathbb{R}$ by
\begin{equation*}
    g(x):=
\begin{cases}
g_R(x), & x\in C_R,\\
0, & x\notin C_R.
\end{cases}
\end{equation*}

\paragraph{Clarification (global domain).}
Classical KAS is stated on compact domains; we therefore interpret ``$g$ is KAS-representable'' as
``$g|_{C_R}$ is KAS-representable after rescaling''. The arbitrary definition of $g$ outside $C_R$ is immaterial
for the $L^p(\mu)$ approximation since $\mu(C_R^c)$ can be made arbitrarily small by Step~1.

\paragraph{Step 5 (combine errors with Minkowski).}
On $\mathbb{R}^n$ we have
\begin{equation*}
    f-g = (f-f_R) + (f_R-\tilde f) + (\tilde f-h) + (h-g).
\end{equation*}
Moreover, $h-g=0$ $\mu$-a.e. because $h=g$ on $C_R$ and both vanish outside $C_R$ under our extensions.
Hence, by Minkowski's inequality (valid for $p\ge 1$) \citep{minkowski1910},
\begin{equation*}
    \|f-g\|_{p}
\le
\|f-f_R\|_{p} + \|f_R-\tilde f\|_{p} + \|\tilde f-h\|_{p}
< \delta+\delta+\delta
=
\varepsilon^{1/p}.
\end{equation*}

Therefore $\mathbb{E}[|f(X)-g(X)|^p]=\|f-g\|_{p}^p<\varepsilon$, as desired.

\medskip
Lemma~\ref{lem:kas_density_lp} is not an identifiability result; it simply supports our use of KAS edge-functional decomposition
as an existence theorem for well-posed surgical edge interventions under a distributional notion of approximation.
\end{proof}

\section{Intervention hierarchy extensions}
\label{app:intervention-hierarchy}

\subsection{Necessary vs.\ sufficient intervention types across the hierarchy}
\label{sec:necessity-vs-sufficiency}

\paragraph{Which intervention tiers are (and are not) needed for identifying $(\alpha,\{\phi_i\},\Gamma)$ when $A$ is unobserved.}
Theorem~\ref{thm:complete-id} shows that to identify the \emph{population} kernels $(\alpha,\{\phi_i\},\Gamma)$ in the sense of conditional distributions, it suffices to use \emph{node} interventions at the context and value levels,
$\mathcal{I}_{\beta\text{-node}}\cup\mathcal{I}_{V\text{-node}}$; edge-message interventions are not required for this kernel-level target.
At the same time, ``necessity'' is relative to the chosen $(\mathcal{I},\Ocal)$ (Def.~\ref{def:I-equiv}); our results establish two concrete barriers and one additional readout requirement:
\begin{enumerate}[label=(\roman*),nosep,leftmargin=*]
    \item If contexts are unobserved and \emph{no} $\beta$-level interventions are available, then $(\alpha,\{\phi_i\},\Gamma)$ are not identifiable beyond context reparameterization (Proposition~\ref{prop:nonid}); $V$-node/\mbox{$V$-edge} interventions in Phase~II cannot break a symmetry that lives in Phase~I.
    \item If adjacency is unobserved, some \emph{structure readout} resource is required to condition on events such as $\{\Pa_A(i)=S\}$; in this paper it is provided by repeated Phase~II probing using $V$-node interventions (Assumption~\ref{ass:struct-readout}), with a sufficient condition stated in Proposition~\ref{prop:instance-id}(i).
    \item Edge-message interventions become essential only when the identification target is \emph{finer} than kernels; for instance, to predict the effects of edge interventions themselves under latent edges (Lemma~\ref{lem:struct_mech_confounding}), or to identify dyadic message primitives in a message-augmented representation (Theorem~\ref{thm:msg-id}).
\end{enumerate}

\begin{table}[t]
\centering
\begin{tabular}{lcccc}
\toprule
Identification target (latent $A$) & $\beta$-node & $V$-node & $\beta$-edge & $V$-edge \\
\midrule
Kernels $(\alpha,\{\phi_i\},\Gamma)$ (Thm.~\ref{thm:complete-id}) & sufficient & sufficient$^\dagger$ & not needed & not needed \\
Value dyadic primitives $\{H^V_{i\leftarrow j}\}$ and slot-response kernels of $F_i$ (Thm.~\ref{thm:msg-id}) & not needed & used$^\dagger$ & not needed & used \\
Context dyadic primitives $\{H^\beta_{i\leftarrow j}\}$ and slot-response kernels of $\Phi_i$ (Thm.~\ref{thm:msg-id}) & used & used$^\dagger$ & used & not needed \\
Realized adjacency $A^\star$ (Prop.~\ref{prop:instance-id}(i)) & not needed & used & not needed & not needed \\
\bottomrule
\end{tabular}
\caption{Summary of which primitive intervention types are used by different identification targets in this paper when adjacency is not directly observed. ``Sufficient/used'' refers to sufficient intervention families under the stated assumptions; ``not needed'' means the primitive does not appear in that result. $^\dagger$ $V$-node interventions enter through the probing-based structure readout protocol (Assumption~\ref{ass:struct-readout}). Necessity in the strict sense is relative to $(\mathcal{I},\Ocal)$; we give explicit impossibility results for missing $\beta$-level access (Prop.~\ref{prop:nonid}) and for predicting edge-intervention effects under latent edges (Lemma~\ref{lem:struct_mech_confounding}).}
\label{tab:intervention-summary}
\end{table}

\paragraph{An intermediate regime: $\beta$-edge interventions without $\beta$-node interventions.}
Proposition~\ref{prop:nonid} rules out identification when $\beta$ is unobserved and no \emph{$\beta$-level} interventions are available.
A natural intermediate experimental regime, relevant in settings where one can perturb a specific communication channel (synapse, interaction) but cannot directly ``set'' a unit's latent type, is to allow $\beta$-edge interventions while disallowing $\beta$-node interventions.
Because an edge intervention replaces an internal message map with an \emph{externally specified} function $\widetilde H$, it need not respect the reparameterization symmetry used in Proposition~\ref{prop:nonid}.
In particular, constant clamps $\widetilde H(\cdot)\equiv m$ erase dependence on the latent context and therefore do not, by themselves, anchor a continuous context coordinate system.
In contrast, non-constant replacements can in principle break the symmetry: under a reparameterization $\gamma$ of the latent context space, the post-intervention message becomes $\widetilde H(\beta_s)$ in one model and $\widetilde H(\gamma(\beta_s))$ in a reparameterized model, which generically induces different interventional measures unless $\widetilde H$ is invariant to $\gamma$.
This suggests that sufficiently rich $\beta$-edge interventions may partially substitute for $\beta$-node interventions for kernel identification at non-root nodes, whereas root contexts remain intrinsically harder to probe in Phase~I because they have no incoming $\beta$-channels.
We leave a formal minimality characterization of intervention families in such restricted regimes to future work.

\begin{assumption}[Edge-intervention richness and channel distinguishability]
\label{ass:edge-richness}
The family $\mathcal{I}$ contains a sufficiently rich class of edge-message replacements (Def.~\ref{def:edge-message-interventions}) to probe a single dyadic channel while holding other inputs fixed. In particular, for each dyad $j\to i$ the intervention family can implement constant ``message clamps'' $\widetilde H(\cdot)\equiv m$ for $m$ in a set with non-empty interior in the message space.

Moreover, fixing all other parent values by $V$-node interventions (Assumption~\ref{ass:parentval}) and conditioning on the event $\{A_{ji}=1\}$, the map
\[
m \;\longmapsto\; \P\!\bigl(V_i \mid \doop_{V}(j\to i;\widetilde H\equiv m),\ \doop(V_{\Pa_A(i)}=v),\ \beta_i=b,\ \Pa_A(i)=S,\ A_{ji}=1\bigr)
\]
is injective on that reachable set (non-degenerate message channel).

\paragraph{Replay variant.} For the replay-based Route~(C) in Theorem~\ref{thm:msg-id}, we additionally assume a pointwise (noise-fixed) version of channel distinguishability: for any fixed $(v,b,S)$ as above (and fixed values of the other message slots induced by $(v,S)$), the map
\[
 m \;\longmapsto\; F_i(\ldots,m,\ldots,u)
\]
 is injective on the reachable set for $\P(U^V_i\mid \beta_i=b,\Pa_A(i)=S)$-almost every $u$.
\end{assumption}

\begin{assumption}[Paired-world replay (idealized)]
\label{ass:paired}
In a paired multi-world experiment, multiple interventions may be applied to the \emph{same} underlying exogenous draw (i.e., the same latent noise realization), so that the only difference between two ``worlds'' is the intervention itself. This assumption is \emph{not} required for Theorem~\ref{thm:complete-id}; it is only used as an \emph{idealized} route for identifying fine-grained dyadic message primitives in Theorem~\ref{thm:msg-id}(C).
\end{assumption}

Theorem~\ref{thm:complete-id} is the \emph{kernel-level} identification result: it establishes identifiability of the causal conditionals $\alpha$, $\P(\beta_i\mid \beta_{\Pa_A(i)},\Pa_A(i))$, and $\P(V_i\mid V_{\Pa_A(i)},\beta_i,\Pa_A(i))$ from ordinary (independently sampled) interventional distributions.  Under Assumption~\ref{ass:struct-readout}, the adjacency is \emph{never} directly observed but is recovered via a probing-based readout procedure (Remark~\ref{rem:probe-readout}), so that events such as $\{\Pa_A(i)=S\}$ may be conditioned on after this recovery.

\paragraph{From kernels to dyadic message mechanisms.}
Theorem~\ref{thm:complete-id} identifies node-level \emph{kernel} (conditional distributions) but does not identify a particular dyadic decomposition of those kernels (Remark~\ref{rem:kas-gauge}).
The next theorem shows that, under stronger experimental access, the dyadic message primitives used to define edge interventions become identifiable (up to gauge).
This refinement is not required for kernel identification, but it clarifies when edge-mechanism questions are learnable from data.

\begin{theorem}[Dyadic message identifiability (three routes)]
\label{thm:msg-id}
Assume the POSCM admits a message-augmented representation (Def.~\ref{def:msg-poscm}). Assume structure readout (Assumption~\ref{ass:struct-readout}), \emph{context readout} (Assumption~\ref{ass:context-readout}), value readout (Assumption~\ref{ass:value-readout}), positivity (Assumption~\ref{ass:positivity}), and edge-intervention richness (Assumption~\ref{ass:edge-richness}). Then the \emph{value} message primitives $\{H^{V}_{i\leftarrow j}\}_{j<i}$ are identifiable on the intervention-reachable support, up to representation gauge (Remark~\ref{rem:kas-gauge}). Moreover, under Routes~(A)--(B) below, for each target node $i$ and dyad $j\to i$, the induced $j$-slot \emph{response kernel}
\[
m \longmapsto \P\!\bigl(V_i \mid \doop_{V}(j\to i;\widetilde H\equiv m),\ \doop(V_{\Pa_A(i)}=v),\ \beta_i=b,\ \Pa_A(i)=S,\ A_{ji}=1\bigr)
\]
is identifiable on the reachable support. (In general this identifies only these pushforward kernels, not the deterministic map $F_i$ separately from the latent noise $U^V_i$; cf.\ Proposition~\ref{prop:instance-id}(iii).)
The message-primitive identification holds under any one of the following routes:
\begin{enumerate}[label=(\Alph*),nosep,leftmargin=*]
    \item \textbf{Joint node controls (no replay).} If value-intervention coverage holds (Assumption~\ref{ass:parentval}), then for each target node $i$ and dyad $j\to i$, combining (i) joint $V$-node interventions that fix all parent values $V_{\Pa_A(i)}$ and (ii) a family of \emph{single-edge} $V$-edge message clamps on $j\to i$ with varying constants $m$ isolates the $j$th dyadic channel across independent experimental units, identifying $H^{V}_{i\leftarrow j}$ and the associated $j$-slot response kernel on the reachable support.
    \item \textbf{Message clamping (no replay).} Under constant clamps $\widetilde H(\cdot)\equiv m$, comparing the baseline distribution at $\doop(V_j=v)$ (holding other parent values fixed) to the family of clamped distributions recovers the unique message value $m=H^{V}_{i\leftarrow j}(v)$ that matches the baseline (uniqueness by Assumption~\ref{ass:edge-richness}), thereby identifying $H^{V}_{i\leftarrow j}(\cdot)$ on the reachable support.
    \item \textbf{Paired-world replay (idealized).} Under paired-world replay (Assumption~\ref{ass:paired}) together with the pointwise distinguishability condition in Assumption~\ref{ass:edge-richness}, replay blocks allow within-unit calibration across clamp values while holding all other inputs and the full exogenous draw fixed, identifying $H^{V}_{i\leftarrow j}$ on the reachable support \emph{without} joint parent-value controls (provided $V_j$ varies across replay blocks over that support).
\end{enumerate}
Analogous statements hold for the \emph{context} message primitives $\{H^{\beta}_{i\leftarrow j}\}$ and the induced slot-response kernels of $\{\Phi_i\}$, replacing $V$-operations with $\beta$-operations (and using context readout to access the relevant conditioning events).
\end{theorem}

\paragraph{Population kernels versus realized instances.}
Theorem~\ref{thm:complete-id} concerns identifiability of population-level kernels from interventional distributions across independently sampled units.
In applications one may also wish to recover the realized adjacency and mechanisms of a \emph{particular} stationary instance.
The following proposition clarifies this distinction and provides a sufficient-condition companion to Assumption~\ref{ass:struct-readout}.

\begin{proposition}[Identifying a realized draw $(A^\star,f^\star)$ from identified kernels]
\label{prop:instance-id}
Theorem~\ref{thm:complete-id} identifies \emph{population-level} kernels. A particular realized draw $A^\star$ and mechanism collection $f^\star$ (Remark~\ref{rem:scm}) is \emph{not} identifiable from population-level interventional distributions alone, since kernels do not reveal which random outcome occurred. However, if the protocol permits repeated probing of the \emph{same} underlying instance and the instance is stationary across probes, then $A^\star$ and (in some regimes) $f^\star$ become identifiable.
\begin{enumerate}[label=(\roman*),nosep,leftmargin=*]
    \item If (a) the instance is stationary across repeated probes with fresh Phase~II value noise and measurement noise resampled i.i.d.\ across probes (Assumption~\ref{ass:struct-readout}), (b) joint $V$-node interventions are available (Assumption~\ref{ass:parentval}), (c) value readout holds (Assumption~\ref{ass:value-readout}), and (d) each true edge is \emph{controlled edge-faithful} in the following sense: whenever $A^\star_{ji}=1$ there exist values $v\neq v'$ and an assignment $v_{-j}$ to $V_{-j}:=(V_k)_{k<i,\,k\neq j}$ such that
\[
\P\!\bigl(V_i \mid \doop(V_j=v,\ V_{-j}=v_{-j})\bigr)\neq
\P\!\bigl(V_i \mid \doop(V_j=v',\ V_{-j}=v_{-j})\bigr),
\]
then $A^\star$ is identifiable in the large-sample limit. Moreover, for any desired error level $\delta>0$, a finite number of probes per tested intervention setting suffices to recover $A^\star$ with probability at least $1-\delta$ under standard finite-sample two-sample distinguishability/testing conditions.
    \item If $\Gamma$ is \emph{deterministic} given $(\beta_i^\star,\Pa_{A^\star}(i))$ (i.e., $f_i=\Gamma(\beta_i,\Pa_A(i))$ with no mechanism-randomness), then once $(A^\star,\beta^\star)$ are identified we have $f_i^\star=\Gamma(\beta_i^\star,\Pa_{A^\star}(i))$.
    \item If $\Gamma$ is \emph{stochastic}, then repeated probes identify the \emph{unit-specific response kernels} $v\mapsto \P(V_i\mid \doop(V_{\Pa_{A^\star}(i)}=v),f_i^\star)$ on the probed domain. Moreover, if the mechanism is distinguishable from its probed response kernel in the sense that the map
\[
f_i \longmapsto \bigl\{\P(V_i\mid \doop(V_{\Pa_{A^\star}(i)}=v),f_i)\bigr\}_{v\in\mathcal{V}_{\mathrm{probe}}}
\]
is injective on the candidate class (modulo observational equivalence on $\mathcal{V}_{\mathrm{probe}}$), then $f_i^\star$ is identified up to that equivalence. Identifying the \emph{structural function} $f_i^\star(\cdot,\cdot)$ itself is strictly stronger and, in general, requires additional noise-model structure (e.g., noise-free mechanisms, known additive noise, monotone structural quantiles, or a parametric mechanism class).
\end{enumerate}
A full statement and proof sketch are given in Appendix~\ref{app:proof:prop:instance-id}.
\end{proposition}

The readout assumptions in Theorem~\ref{thm:complete-id} directly address the barriers identified in Proposition~\ref{prop:nonid} (context readout) and Lemma~\ref{lem:struct_mech_confounding} (structure readout).

\section{Identifiability Proofs}
\label{app:proofs}

\paragraph{Readout and coverage assumptions for Theorem~\ref{thm:complete-id}.}
\label{app:kernel-assumptions}

\begin{assumption}[Experimental structure readout (latent adjacency)]
\label{ass:struct-readout}
The adjacency $A$ is latent (not directly observed).
Instead, the protocol provides a structure readout mechanism \emph{via multiple intervention experiments}:
each experimental unit may be probed repeatedly in Phase~II while its realized $(A^\star,\beta^\star,f^\star)$ remains stationary across probes. Across probes, Phase~II exogenous randomness (value noise) and measurement noise are resampled i.i.d., so repeated probes yield i.i.d.\ draws conditional on $(A^\star,\beta^\star,f^\star)$. The experimenter may apply joint $V$-node interventions (Assumption~\ref{ass:parentval}) across probes.
These repeated probes suffice to identify the realized adjacency $A^\star$ in the large-sample limit (a sufficient condition is given in Proposition~\ref{prop:instance-id}(i)).
Whenever results below condition on events involving $A$ (e.g., $\{\Pa_A(i)=S\}$), this is understood in terms of the identified adjacency returned by the probing procedure.
\end{assumption}

\begin{remark}[Structure readout via repeated probing]
\label{rem:probe-readout}
Assumption~\ref{ass:struct-readout} rules out direct observation of $A$.
Instead, it posits that adjacency can be recovered from \emph{multiple intervention experiments} on a stationary instance.
A sufficient condition is that value-intervention coverage holds and each realized edge is edge-faithful in a \emph{controlled} sense (Proposition~\ref{prop:instance-id}(i)), in which case repeated Phase~II probes identify the realized adjacency $A^\star$ (Proposition~\ref{prop:instance-id}(i)).
Proposition~\ref{prop:instance-id}(i) is a sufficient-condition companion to Assumption~\ref{ass:struct-readout} and is proved independently of Theorem~\ref{thm:complete-id}.
Kernel identification in Theorem~\ref{thm:complete-id} still relies on sampling \emph{independent} instances under interventions; repeated probing of a single instance identifies only $A^\star$ and, when $\Gamma$ is stochastic, unit-specific response kernels (Proposition~\ref{prop:instance-id}(iii)).
\end{remark}

\begin{assumption}[Context readout]
\label{ass:context-readout}
Contexts are observed (possibly noisily) through $\betatilde=\Ocal^{\beta}(\beta,U^{\Ocal}_\beta)$ in a way that makes $\beta$ identifiable up to any explicitly stated equivalence (e.g., label permutation).
\end{assumption}

\begin{assumption}[Value readout]
\label{ass:value-readout}
Endogenous values are observed (possibly noisily) through $\Vtilde=\Ocal^{V}(V,U^{\Ocal}_V)$ in a way that makes $V$ identifiable up to any explicitly stated equivalence.
\end{assumption}

\begin{assumption}[Context-intervention coverage]
\label{ass:coverage}
For each node $j$, the family of $\beta$-node interventions can set $\beta_j$ to values in a set with non-empty interior (continuous case) or can visit all labels (discrete case). Moreover, the protocol supports \emph{joint} $\beta$-node interventions: for any subset $S\subseteq\{0,\ldots,N{-}1\}$ and any tuple $b_S$ in the intervention-reachable set, the intervention $\doop(\beta_S=b_S)$ is admissible.
\end{assumption}

\begin{assumption}[Value-intervention coverage]
\label{ass:parentval}
For each node $j$, the family of $V$-node interventions can set $V_j$ to values in a set with non-empty interior (continuous case) or over full support (discrete case). Moreover, the protocol supports \emph{joint} $V$-node interventions: for any subset $S\subseteq\{0,\ldots,N{-}1\}$ and any tuple $v_S$ in the intervention-reachable set, the intervention $\doop(V_S=v_S)$ is admissible. In particular, for each node $i$ and each realized parent set $S\subseteq\{j:j<i\}$, the family of $V$-node interventions can set $(V_j)_{j\in S}$ jointly to any tuple $v_S$ in a set with non-empty interior (continuous case) or over full support (discrete case). Interventions at the $\beta$ level and the $V$ level may be combined in the same experiment (Phase~I + Phase~II).
\end{assumption}

\begin{assumption}[Positivity]
\label{ass:positivity}
On the intervention-reachable support, for all dyads $(j,i)$ with $j<i$ and all reachable values $b$,
\[
0 < \P\!\bigl(A_{ji}=1 \mid \doop(\beta_j=b)\bigr) < 1.
\]
\end{assumption}

\paragraph{Scope and assumption boundary.}
The results in this section are \emph{conditional identifiability} statements: they show that,
assuming the experimental capabilities in Assumptions~\ref{ass:struct-readout}--\ref{ass:paired},
the corresponding kernels/messages/mechanisms are identifiable from the stated interventional
distributions.  In particular, the readout/coverage/injectivity conditions are not derived from the
bare POSCM generative semantics; they encode additional experimental structure that breaks the
symmetries highlighted in Proposition~\ref{prop:nonid} and Lemma~\ref{lem:struct_mech_confounding}.

\begin{table}[t]
\centering
\begin{tabular}{p{0.31\linewidth}p{0.61\linewidth}}
\toprule
Relaxed condition & Consequence \\
\midrule
Limited intervention coverage & Kernels are identified only on the intervention-reachable support. \\
No positivity in a stratum & Zero-mass parent/context strata remain unlearned. \\
Noisy context/value readout & Kernel estimates inherit decoder error; identification is up to the readout equivalence. \\
Missing $A$ readout/probing & Parent-set-specific kernels collapse into mixtures over latent parent sets. \\
No edge-message access & Kernel targets may remain identifiable, but dyadic message primitives and edge-intervention counterfactuals require additional channel access. \\
\bottomrule
\end{tabular}
\caption{How the sufficient assumptions in Theorem~\ref{thm:complete-id} relax. The result is local on the support reached by the available intervention family.}
\label{tab:assumption-relaxations}
\end{table}

\paragraph{Finite-strata sample scaling.}
For finite or smoothed discrete strata, the usual plug-in multinomial scaling gives an order-of-magnitude guide. If $n$ is the effective intervention budget, $\rho$ is the minimum probability of a reachable conditioning stratum, and $d_{\mathrm{eff}}$ is the number of free kernel parameters in that stratum, then one expects
\[
\mathbb{E}\!\left[\mathrm{KL}(K\|\widehat K)\right]
=
O\!\left(\frac{d_{\mathrm{eff}}}{n\rho}\right)
+
\varepsilon_{\mathrm{readout}}
+
\varepsilon_{\mathrm{probe}},
\]
up to constants and logarithmic factors. The final terms denote context/value readout error and structure-probing error. This is not a new identification theorem; it only summarizes the standard finite-strata cost of estimating the kernels once the theorem's experimental interface is available.


\subsection{Proof of Proposition~\ref{prop:nonid}}
\label{app:proof:prop:nonid}

\begin{proof}
We show that when contexts are unobserved and no $\beta$-level interventions are available, the model admits a context-reparameterization symmetry that cannot be broken by $V$-level data.

Fix any POSCM $\mathcal{M}$ and any diffeomorphism $\gamma:\mathcal{B}\to\mathcal{B}$.
We construct a second POSCM $\mathcal{M}'$ with different $(\P(A\mid \beta),\{\phi_i\},\Gamma)$ that is $(\mathcal{I},\Ocal)$-equivalent to $\mathcal{M}$ for $\mathcal{I}=\mathcal{I}_{\mathrm{obs}}\cup \mathcal{I}_{V\text{-node}}\cup \mathcal{I}_{V\text{-edge}}$.

\textbf{Construction of $\mathcal{M}'$.}
Let $\beta'_i := \gamma(\beta_i)$ for all $i$ and define $\mathcal{M}'$ via pushforward/composition:
\begin{align}
    \P_{\beta'_1} &:= \gamma_{\#} \P_{\beta_1}, \label{eq:pushforward-root-new}\\
    \P'(A_{ji} = 1 \mid \beta'_j) &:= \P(A_{ji} = 1 \mid \gamma^{-1}(\beta'_j)), \label{eq:kernel-reparam-new}\\
    \phi'_i(\{\beta'_k\}_{k \in \mathrm{Pa}(i)}, U^\beta_i) &:= \gamma\!\left(\phi_i(\{\gamma^{-1}(\beta'_k)\}_{k \in \mathrm{Pa}(i)}, U^\beta_i)\right), \label{eq:context-prop-reparam-new}\\
    \P'(f_i \mid \beta'_i, \mathrm{Pa}_A(i)) &:= \P(f_i \mid \gamma^{-1}(\beta'_i), \mathrm{Pa}_A(i)). \label{eq:mech-assign-reparam-new}
\end{align}
All other components (value-noise measures, etc.) are unchanged, and the measurement model $\Ocal$ is held fixed (Def.~\ref{def:measurement-model}).

\textbf{Equality of observable interventional measures.}
Consider any intervention $\iota\in\mathcal{I}$, which acts only at the $V$ level (Phase~II).
By construction \eqref{eq:pushforward-root-new}--\eqref{eq:mech-assign-reparam-new}, $\mathcal{M}$ and $\mathcal{M}'$ induce the same interventional measure for $(A,V)$ after marginalizing out $\beta$ (since $\beta'$ is a reparameterization of $\beta$ and all kernels/mechanism-assignment factors are composed accordingly).
Under the factorized measurement model of Def.~\ref{def:measurement-model}, the observable channels $\tilde A=\Ocal^{A}(A,U^{\Ocal}_A)$ and $\tilde V=\Ocal^{V}(V,U^{\Ocal}_V)$ do not depend on $\beta$.
When $\beta$ is unobserved, $\Ocal^{\beta}$ outputs nothing, so the induced interventional distributions over observables $(\tilde A,\tilde V)$ coincide under $\mathcal{M}$ and $\mathcal{M}'$ for all $\iota\in\mathcal{I}$.
Therefore $\P(A\mid \beta)$, $\{\phi_i\}$, and $\Gamma$ are not $\mathcal{I}$-identifiable beyond context reparameterization.
\end{proof}

\subsection{Proof of Lemma~\ref{lem:struct_mech_confounding}}
\label{app:proof:lem:struct_mech_confounding}

\begin{proof}
We construct an explicit counterexample showing that two POSCMs can agree on all node-level interventional distributions yet disagree on $V$-edge intervention counterfactuals.

\textbf{Setup.} Consider a two-node ordered POSCM with $\tau(1) < \tau(2)$ and a single potential dyad $1 \to 2$. Let contexts be degenerate (omit $\beta$). Let $V_1 \in \{0,1\}$ be an endogenous variable. Let $A_{12} \sim \mathrm{Bern}(p)$ be unobserved.

Define the value mechanism at node $2$ as follows:
\begin{itemize}[nosep]
    \item If $A_{12} = 0$: node $2$ has no parents and outputs $V_2 \sim \mathrm{Bern}(1/2)$.
    \item If $A_{12} = 1$: node $2$ has parent set $\{1\}$ and outputs $V_2 \sim \mathrm{Bern}(q_{V_1})$ for parameters $q_0, q_1 \in (0,1)$.
\end{itemize}

\textbf{Node-level interventional distributions.} For each node intervention $\mathrm{do}(V_1 = v)$, $v \in \{0,1\}$, we have:
\begin{align*}
    \P(V_2 = 1 \mid \mathrm{do}(V_1 = v)) = (1-p) \cdot \tfrac{1}{2} + p \cdot q_v.
\end{align*}

\textbf{Model $\mathcal{M}$.} Fix $p = \tfrac{1}{2}$ and $(q_0, q_1) = (0.2, 0.8)$. This yields:
\begin{align*}
    \P_{\mathcal{M}}(V_2 = 1 \mid \mathrm{do}(V_1 = 0)) &= 0.5 \cdot 0.5 + 0.5 \cdot 0.2 = 0.35, \\
    \P_{\mathcal{M}}(V_2 = 1 \mid \mathrm{do}(V_1 = 1)) &= 0.5 \cdot 0.5 + 0.5 \cdot 0.8 = 0.65.
\end{align*}

\textbf{Model $\mathcal{M}'$.} Choose a different edge probability $p' = 0.8$ and define $(q'_0, q'_1)$ by solving:
\begin{equation*}
    (1 - p') \cdot \tfrac{1}{2} + p' \cdot q'_v = (1 - p) \cdot \tfrac{1}{2} + p \cdot q_v, \quad v \in \{0,1\}.
\end{equation*}

Substituting:
\begin{align*}
    0.2 \cdot 0.5 + 0.8 \cdot q'_0 &= 0.35 \quad \Rightarrow \quad q'_0 = (0.35 - 0.1)/0.8 = 0.3125, \\
    0.2 \cdot 0.5 + 0.8 \cdot q'_1 &= 0.65 \quad \Rightarrow \quad q'_1 = (0.65 - 0.1)/0.8 = 0.6875.
\end{align*}
Both $q'_0, q'_1 \in (0,1)$, so $\mathcal{M}'$ is a valid POSCM.

\textbf{Node-level equivalence.} By construction:
\begin{equation*}
    \P_{\mathcal{M}}(V_2 = 1 \mid \mathrm{do}(V_1 = v)) = \P_{\mathcal{M}'}(V_2 = 1 \mid \mathrm{do}(V_1 = v)) \quad \text{for } v \in \{0,1\}.
\end{equation*}

\textbf{Message-augmented realization (Def.~\ref{def:msg-poscm}).}
To match the edge-intervention definition in Def.~\ref{def:edge-message-interventions},
realize the above mechanism in message-augmented form with message dimension $d_2=2$.
Let $U^V_2 \sim \mathrm{Unif}(0,1)$ be exogenous variable and define the (baseline) value-message primitive
\[
H^{V}_{2\leftarrow 1}(v) := (1,q_v),
\qquad
M^{V}_{2\leftarrow 1} := A_{12} H^{V}_{2\leftarrow 1}(V_1) = (A_{12},A_{12}q_{V_1}),
\]
where the first coordinate records whether the edge is present.
Define the aggregator
\[
F_2\big((m_1,m_2),u\big) := \mathbf{1}\Big\{u < (1-m_1)\tfrac12 + m_2\Big\}.
\]
Then $V_2 = F_2(M^{V}_{2\leftarrow 1},U^V_2)$ yields $V_2\sim\mathrm{Bern}(1/2)$ when $A_{12}=0$
and $V_2\sim\mathrm{Bern}(q_{V_1})$ when $A_{12}=1$, as above.

\textbf{$V$-edge intervention counterfactuals differ.}
Consider the $V$-edge message intervention of Def.~\ref{def:edge-message-interventions}
\begin{equation*}
    \iota := \doop_{V}(1 \to 2;\,\widetilde H^{V}_{2\leftarrow 1}),
\qquad
\widetilde H^{V}_{2\leftarrow 1}(v) := (1,v).
\end{equation*}
Equivalently, when $A_{12}=1$ this replaces $q_v$ by $\tilde q_v=v$ and hence forces $V_2=V_1$.
Because messages are gated by $A_{12}$, $\iota$ is a no-op when $A_{12}=0$.
Under $\iota$, for $v \in \{0,1\}$:
\begin{equation*}
    \P\!\big(V_2 = 1 \mid \iota,\, \doop(V_1 = v)\big)
= (1-p)\cdot\tfrac12 + p\cdot v.
\end{equation*}

In $\mathcal{M}$ (with $p = 0.5$):
\begin{align*}
    \P_{\mathcal{M}}(V_2 = 1 \mid \iota,\, \doop(V_1 = 0)) &= 0.5 \cdot 0.5 + 0.5 \cdot 0 = 0.25, \\
    \P_{\mathcal{M}}(V_2 = 1 \mid \iota,\, \doop(V_1 = 1)) &= 0.5 \cdot 0.5 + 0.5 \cdot 1 = 0.75.
\end{align*}

In $\mathcal{M}'$ (with $p' = 0.8$):
\begin{align*}
    \P_{\mathcal{M}'}(V_2 = 1 \mid \iota,\, \doop(V_1 = 0)) &= 0.2 \cdot 0.5 + 0.8 \cdot 0 = 0.10, \\
    \P_{\mathcal{M}'}(V_2 = 1 \mid \iota,\, \doop(V_1 = 1)) &= 0.2 \cdot 0.5 + 0.8 \cdot 1 = 0.90.
\end{align*}

Since $p \neq p'$, the post-edge-intervention distributions differ between $\mathcal{M}$ and $\mathcal{M}'$.

\textbf{Conclusion.} Node-level interventional data do not determine $V$-edge intervention counterfactuals in the nonparametric latent-edge setting. The underlying ambiguity is \emph{structure-mechanism confounding}: a model with frequent edges and weak edge-gated mechanisms can be observationally equivalent (at the node level) to a model with rare edges and strong edge-gated mechanisms, yet these models make different predictions about what would happen under edge-level interventions.
\end{proof}

\subsection{Proof of Theorem~\ref{thm:complete-id} (Kernel identifiability; no paired-world replay)}
\label{app:proof:thm:complete-id}

\begin{proof}
Let $\mathcal{M}$ and $\mathcal{M}'$ be two POSCMs that are $(\mathcal{I}_{\mathrm{kern}},\Ocal)$-equivalent, where $\mathcal{I}_{\mathrm{kern}} := \mathcal{I}_{\mathrm{obs}} \cup \mathcal{I}_{\beta\text{-node}} \cup \mathcal{I}_{V\text{-node}}$. By Assumptions~\ref{ass:context-readout} and~\ref{ass:value-readout}, equality of interventional distributions over the observable readouts implies equality of the corresponding interventional distributions over $(\beta,V)$ (up to any explicitly stated measurement equivalences). Moreover, by Assumption~\ref{ass:struct-readout} the adjacency $A$ is identifiable under each intervention via the probing-based structure readout (Remark~\ref{rem:probe-readout}), so the interventional distribution of $A$ is identified as well. We therefore reason below as if $(A,\beta,V)$ were available, with $A$ obtained via the probing readout.

\paragraph{Proof template (decoder/inversion $\Rightarrow$ identifiability).}
For each target kernel family in Theorem~\ref{thm:complete-id}, the goal is to show that it is a
(deteministic) function of the restricted interventional distribution family
$\{\P(A,\beta,V\mid \iota)\}_{\iota\in\mathcal{I}_{\mathrm{kern}}}$.
Equivalently, we exhibit an explicit \emph{decoder} (an inversion map) from these distributions to the target
kernel.  Therefore, if $\mathcal{M}$ and $\mathcal{M}'$ are
$(\mathcal{I}_{\mathrm{kern}},\Ocal)$-equivalent, they induce the same restricted interventional distributions,
hence the same decoder output, and thus the same target kernel (up to the measurement equivalences
in Assumptions~\ref{ass:context-readout}--\ref{ass:value-readout}).

We prove identifiability of each kernel family from interventional distributions collected from \emph{independent experimental units} (no paired-world replay).

\paragraph{(i) Structure kernel $\alpha$.}
Fix a source node $j$ and a reachable value $b$. Under the intervention $\doop(\beta_j=b)$ (admissible by Assumption~\ref{ass:coverage}), the Phase~I edge-formation step generates the outgoing edge vector $A_{j,>j}$ using the structure kernel evaluated at $\beta_j=b$. Therefore the interventional distribution
\[
\P(A_{j,>j}\mid \doop(\beta_j=b))
\]
is exactly the interventional supervising measure $\mu_j^{\beta\text{-do}}(b)$ appearing in Theorem~\ref{thm:complete-id}(i). By structure readout (Assumption~\ref{ass:struct-readout}), this measure is identified from the observable interventional distribution for each reachable $b$. Varying $b$ over the intervention-reachable set identifies the interventional supervising measures $\{\mu_j^{\beta\text{-do}}(\cdot)\}_j$ and, in particular, each dyadwise marginal $b\mapsto \P(A_{ji}=1\mid \doop(\beta_j=b))$ for $i>j$. Hence $\alpha$ is $(\mathcal{I}_{\mathrm{kern}},\Ocal)$-identifiable on the reachable support.

\paragraph{(ii) Context propagation kernels induced by $\{\phi_i\}$.}
Fix a node $i$ and a parent set $S$ with $\P(\Pa_A(i)=S\mid \iota)>0$ for some $\iota\in \mathcal{I}_{\mathrm{kern}}$. Consider a joint $\beta$-node intervention $\doop(\beta_S=b_S)$, admissible by Assumption~\ref{ass:coverage}. Under ordered generation, the Phase~I context mechanism generates
\[
\beta_i \;=\; \phi_i\!\bigl(\beta_S,\,U^\beta_i\bigr)
\qquad\text{on the event }\{\Pa_A(i)=S\}.
\]
Thus, conditioning on $\Pa_A(i)=S$ (which is identifiable by Assumption~\ref{ass:struct-readout}) and on the intervention value $\beta_S=b_S$, the conditional interventional distribution
\[
\P\!\bigl(\beta_i \mid \doop(\beta_S=b_S),\ \Pa_A(i)=S\bigr)
\]
coincides with the kernel $K^\beta_{i,S}(\cdot\mid b_S)$ in Theorem~\ref{thm:complete-id}(ii). By context readout (Assumption~\ref{ass:context-readout}) and structure readout (Assumption~\ref{ass:struct-readout}), this conditional distribution is identified from interventional data for each reachable $b_S$; varying $b_S$ over the reachable set identifies $K^\beta_{i,S}$ on its reachable support. Equivalently, $\P(\beta_i\mid \beta_{\Pa_A(i)},\Pa_A(i))$ is identified (up to any context reparameterization/gauge equivalences discussed in the main text when interventions are not anchored).

No replay is required: conditioning on identified $(A,\beta_S)$ across independent units plays the same isolating role that replay played in the original paired-world argument.

\paragraph{(iii) Endogenous-variable mechanism kernels induced by $\Gamma$.}
Fix a node $i$ and a parent set $S$ in the support of $\Pa_A(i)$ under $\mathcal{I}_{\mathrm{kern}}$. Consider a $V$-node intervention $\doop(V_S=v_S)$, admissible by Assumption~\ref{ass:parentval}. Since $V$-node interventions occur in Phase~II, they do not alter Phase~I variables $(A,\beta)$, so we may condition on $\beta_i=b_i$ and $\Pa_A(i)=S$ (identifiable by Assumptions~\ref{ass:struct-readout} and~\ref{ass:context-readout}).

Under the Phase~II value-generation step, conditional on $\beta_i=b_i$ and $\Pa_A(i)=S$, the node-$i$ mechanism is sampled via $\Gamma$ and then applied to the intervened parent values:
\[
f_i \sim \P(\,\cdot \mid \beta_i=b_i,\Pa_A(i)=S),\qquad
V_i = f_i(v_S,U^V_i).
\]
Therefore the conditional interventional distribution
\[
\P\!\bigl(V_i \mid \doop(V_S=v_S),\ \beta_i=b_i,\ \Pa_A(i)=S\bigr)
\]
coincides with $K^V_{i,S}(\cdot\mid v_S,b_i)$ in Theorem~\ref{thm:complete-id}(iii). By value readout (Assumption~\ref{ass:value-readout}) together with structure/context readout, this conditional distribution is identified for each reachable $(v_S,b_i)$; varying $(v_S,b_i)$ over the intervention-reachable set identifies $K^V_{i,S}$ on its reachable support. Equivalently, $\P(V_i\mid V_{\Pa_A(i)},\beta_i,\Pa_A(i))$ is identified. Note that if $\Gamma$ is stochastic, $K^V_{i,S}$ is a \emph{population} conditional distribution that averages over the mechanism draw $f_i$; repeated probing of a single stationary instance instead identifies the corresponding \emph{unit-specific} response kernel (Proposition~\ref{prop:instance-id}(iii)).

Again, no replay is required: the causal isolation comes from (a) ordered generation, and (b) the ability to fix parent values and condition on $\Pa_A(i)$ across independent experimental units.
\end{proof}

\subsection{Proof of Theorem~\ref{thm:msg-id} (Dyadic message identifiability)}
\label{app:proof:thm:msg-id}
\begin{proof}
We sketch the argument for the \emph{value} message primitives; the context case is analogous.

Fix a target node $i$ and an upstream node $j<i$. In a message-augmented representation (Def.~\ref{def:msg-poscm}), the value-update admits a decomposition of the form
\[
M^{V}_{i\leftarrow k} \,=\, A_{ki}\,H^{V}_{i\leftarrow k}(V_k),
\qquad
V_i \,=\, F_i\!\bigl(\{M^{V}_{i\leftarrow k}\}_{k<i},\,U^V_i\bigr),
\]
up to representation gauge (Remark~\ref{rem:kas-gauge}).

Throughout, we condition on a realized parent set $S=\Pa_A(i)$ and on a fixed context value $\beta_i=b$; these conditioning events are operational under Assumptions~\ref{ass:struct-readout} and~\ref{ass:context-readout}.

\paragraph{Route A (joint node controls; no replay).}
Assume value-intervention coverage (Assumption~\ref{ass:parentval}) and edge-intervention richness (Assumption~\ref{ass:edge-richness}).
Fix a realized parent set $S=\Pa_A(i)$ and fix parent values $v_S$ via a joint $V$-node intervention $\doop(V_S=v_S)$. By structure readout we may condition on $\{\Pa_A(i)=S\}$ and (if needed) on $\{A_{ji}=1\}$.

Under the baseline (no edge intervention), the $j$th message slot equals $m^\star:=H^{V}_{i\leftarrow j}(v_j)$, and all other slots are fixed at $\{H^{V}_{i\leftarrow k}(v_k)\}_{k\in S\setminus\{j\}}$ or $0$ for $k\notin S$. Therefore, for fixed $(v_S,b,S)$ the conditional distribution of $V_i$ is the pushforward of $U^V_i$ through the (unknown) aggregator at message value $m^\star$:
\[
\P\!\bigl(V_i \mid \doop(V_S=v_S),\ \beta_i=b,\ \Pa_A(i)=S,\ A_{ji}=1\bigr)
\;=\;
\P\!\bigl(F_i(\ldots,m^\star,\ldots,U^V_i)\mid \beta_i=b,\ \Pa_A(i)=S\bigr).
\]

Now apply a $V$-edge message intervention on $j\to i$ (Def.~\ref{def:edge-message-interventions}) that replaces $H^{V}_{i\leftarrow j}$ by a \emph{known} constant clamp $\widetilde H(\cdot)\equiv m$. Under the same joint parent-value control and the same conditioning on $(\beta_i=b,\Pa_A(i)=S,A_{ji}=1)$, the resulting conditional distribution is
\[
\P\!\bigl(V_i \mid \doop_{V}(j\to i;\widetilde H\equiv m),\ \doop(V_S=v_S),\ \beta_i=b,\ \Pa_A(i)=S,\ A_{ji}=1\bigr)
\;=\;
\P\!\bigl(F_i(\ldots,m,\ldots,U^V_i)\mid \beta_i=b,\ \Pa_A(i)=S\bigr).
\]
In particular, taking $m=m^\star$ reproduces the baseline conditional distribution, since $F_i$ depends on the $j$th input only through the resulting message value.
Varying $m$ over a rich set identifies the induced \emph{response kernel} $m\mapsto \P(V_i\mid \doop_V(j\to i;\widetilde H\equiv m),\dots)$ on the reachable support.
By injectivity in Assumption~\ref{ass:edge-richness}, the baseline distribution matches \emph{exactly one} member of this family, so $m^\star$ (and hence $H^{V}_{i\leftarrow j}(v_j)$) is uniquely identified on the reachable support (up to gauge).
Repeating the argument for varying $v_j$ (via $V$-node interventions) identifies the function $H^{V}_{i\leftarrow j}(\cdot)$ on the reachable support, and iterating over dyads identifies all value message primitives.

\paragraph{What is (and is not) identified.} The same experiments identify the induced response kernels of the aggregator in the relevant slot. In general, they do \emph{not} identify the deterministic map $F_i$ separately from the latent noise $U^V_i$ without additional noise-model structure (cf.\ Proposition~\ref{prop:instance-id}(iii)).

\paragraph{Route B (message clamping; no replay).}
This is the matching step in Route~A emphasized as a calibration procedure: for fixed $(v_S,b,S)$, the family of clamped interventional distributions parameterized by $m$ provides a ``response curve'' for the $j$th slot, and the baseline distribution at $\doop(V_j=v)$ selects the unique $m=H^{V}_{i\leftarrow j}(v)$ on the reachable support.

\paragraph{Route C (paired-world replay; idealized).}
Under paired-world replay (Assumption~\ref{ass:paired}), we may evaluate a baseline world and multiple $V$-edge clamps on $j\to i$ on the \emph{same} exogenous draw.
Within a replay block $\omega$, all non-descendants of the intervention target (including $A$, all $\beta$, and all upstream $V_k$ for $k<i$) are identical across worlds, and the realized value noise $U^V_i(\omega)$ is shared.
Fix a block with $A_{ji}=1$ and $\Pa_A(i)=S$, and write $v_j:=V_j(\omega)$ and $m^\star:=H^{V}_{i\leftarrow j}(v_j)$. Then, for each clamp value $m$ we observe
\[
V_i^{(m)}(\omega)\;=\;F_i(\ldots,m,\ldots,U^V_i(\omega)).
\]
By the \emph{pointwise} distinguishability condition in Assumption~\ref{ass:edge-richness} (Replay variant), the map $m\mapsto V_i^{(m)}(\omega)$ is injective for almost every block, so matching the baseline value $V_i^{(\mathrm{base})}(\omega)=V_i^{(m^\star)}(\omega)$ to the clamped curve identifies $m^\star$ within that block.
Repeating across replay blocks whose $V_j$ values range over the intervention-reachable support (either by natural variation or by intervening on $V_j$) identifies $H^{V}_{i\leftarrow j}(\cdot)$ pointwise on that support, up to gauge, \emph{without} requiring joint interventions to fix the other parent values.

This completes the proof sketch for value messages; the context-message case is analogous.
\end{proof}

\subsection{Proof of Proposition~\ref{prop:instance-id} (Instance-level identification)}
\label{app:proof:prop:instance-id}

\begin{proof}
We justify each item of Proposition~\ref{prop:instance-id}.

\paragraph{(i) Identifying $A^\star$ by repeated probing (stationary instance).}
Fix a dyad $j<i$. By joint value interventions (Assumption~\ref{ass:parentval}), we may set $V_j=v$ while clamping the other potential parents $V_{-j}:=(V_k)_{k<i,\,k\neq j}$ to an arbitrary assignment $v_{-j}$, i.e., apply $\doop(V_j=v,\ V_{-j}=v_{-j})$. By Assumption~\ref{ass:struct-readout}, repeated probes under a fixed intervention resample Phase~II value noise and measurement noise i.i.d., so we obtain i.i.d.\ samples from the induced conditional distribution of $V_i$ (or its readout, by Assumption~\ref{ass:value-readout}).

If $A^\star_{ji}=0$, then the Phase~II structural equation for $V_i$ does not take $V_j$ as an argument; once $V_{-j}$ is clamped, changing $v$ cannot change the distribution of $V_i$. If instead $A^\star_{ji}=1$, controlled edge-faithfulness provides values $v\neq v'$ and some clamp $v_{-j}$ such that the induced distributions under $\doop(V_j=v,\ V_{-j}=v_{-j})$ and $\doop(V_j=v',\ V_{-j}=v_{-j})$ differ.

With sufficiently many repeated probes per intervention setting, empirical convergence together with any consistent two-sample test distinguishes equality versus inequality of these induced distributions, yielding a consistent decision rule for the edge indicator $A^\star_{ji}$. Applying this decision over all dyads recovers $A^\star$.
\paragraph{(ii) Identifying $f^\star$ when $\Gamma$ is deterministic.}
If $\Gamma$ is deterministic given $(\beta_i,\Pa_A(i))$, then once $\beta^\star$ and $A^\star$ are identified we have
\[
f_i^\star = \Gamma(\beta_i^\star,\Pa_{A^\star}(i)).
\]
No additional data beyond identifying $(A^\star,\beta^\star)$ are required.

\paragraph{(iii) Stochastic $\Gamma$: identifying the realized response kernel vs.\ the structural function.}
If $\Gamma$ is stochastic, then $f_i^\star$ is a latent draw that is fixed across probes. For any probed parent-value assignment $v$ (implemented via $\doop(V_{\Pa_{A^\star}(i)}=v)$), repeated probes yield i.i.d.\ samples from
\[
\P\!\bigl(V_i \mid \doop(V_{\Pa_{A^\star}(i)}=v),\,f_i^\star\bigr),
\]
hence the unit-specific response kernel $v\mapsto \P(V_i\mid \doop(V_{\Pa_{A^\star}(i)}=v),f_i^\star)$ is identified on the probed domain in the large-sample limit.


We can make the ``observational equivalence on $\mathcal{V}_{\mathrm{probe}}$'' precise by defining
\[
f_i \sim_{\mathrm{probe}} f_i'
\quad\Longleftrightarrow\quad
\forall v\in\mathcal{V}_{\mathrm{probe}},\;
\P(V_i\mid \doop(V_{\Pa_{A^\star}(i)}=v),f_i)
=
\P(V_i\mid \doop(V_{\Pa_{A^\star}(i)}=v),f_i').
\]
The condition is that the map
\(f_i \mapsto \{ \P(V_i\mid \doop(V_{\Pa_{A^\star}(i)}=v),f_i)\}_{v\in\mathcal{V}_{\mathrm{probe}}}\)
is injective modulo $\sim_{\mathrm{probe}}$, and the conclusion is that the identified response kernel pins
down $f_i^\star$ up to $\sim_{\mathrm{probe}}$.

Finally, identifying the deterministic map $f_i^\star(\cdot,\cdot)$ itself (separating it from value noise) is strictly stronger: without additional restrictions on how noise enters, many different deterministic functions can induce the same conditional distribution. Standard sufficient conditions include noise-free mechanisms, known additive noise, monotone structural quantiles, or a known parametric mechanism family; in these cases $f_i^\star$ becomes identifiable from the collection of interventional response distributions.
\end{proof}

\begin{table}[t]
\centering
\begin{tabular}{p{0.23\linewidth}p{0.31\linewidth}p{0.36\linewidth}}
\toprule
Primitive & POSCM semantics & Scientific reading \\
\midrule
$\doop(V_j=v)$ & Clamp a Phase-II value; keep $(A,\beta,f)$ fixed & Voltage clamp, expression/input clamp, or activity clamp. \\
$\doop(\beta_j=b)$ & Set pre-formation context; regenerate downstream structure/mechanisms & Developmental, genetic, cell-fate, eccentricity, or accessibility manipulation. \\
$\doop_V(j\to i;\widetilde H)$ & Replace one value-message channel & Synapse/channel conductance or dyadic transfer perturbation. \\
$\doop_\beta(s\to t;\widetilde H)$ & Replace one context-message channel & Local perturbation of a developmental or signaling channel. \\
\bottomrule
\end{tabular}
\caption{Scientific readings of POSCM intervention primitives. A $\beta$-intervention is a pre-formation manipulation; it should not be read as acutely turning an adult cell of one type into another.}
\label{tab:intervention-examples}
\end{table}

\section{Gene-regulatory Network Experiments}
\label{app:grn-analogue}

A second POSCM instantiation is a feedforward gene-regulatory network (GRN). Here $\beta_i$ denotes cell state or chromatin accessibility, $A_{ji}$ denotes whether transcription factor $j$ regulates promoter $i$, $\Gamma$ assigns regulatory parameters such as Hill coefficients, activation/repression signs, and binding strengths, and $V_i$ is gene expression. A $V$-node intervention corresponds to an expression, production, or input clamp; a $\beta$-node intervention corresponds to a pre-formation accessibility or cell-state manipulation; and a $V$-edge intervention corresponds to changing one TF--promoter regulatory channel.

To check that the retina experiments are not artifacts of a self-authored simulator, we repeated the three validation experiments on top of SERGIO (\texttt{sergio-scsim} 1.9.3), an external single-cell expression simulator whose stochastic differential equations use Hill-type regulatory response functions \citep{dibaeinia2020sergio}. The system is feedforward: transcription factors are generated before their targets, so the ordered-generation semantics is the same as in the retina experiments. The experiments below are controlled validations of POSCM identifiability predictions, not a full noisy-data GRN-recovery pipeline.

\begin{table}[t]
\centering
\begin{tabular}{p{0.28\linewidth}p{0.66\linewidth}}
\toprule
GRN object & POSCM / causal interpretation \\
\midrule
Gene expression & Continuous endogenous variable $V_i$. \\
Transcription factor (TF) & Upstream parent variable regulating a target gene. \\
Regulatory edge & Directed structural contribution, represented by $A_{ji}=1$. \\
Hill parameters $(k,h,n)$ & Mechanism parameters assigned by $\Gamma$: strength, half-activation threshold, and cooperativity. \\
Activator / repressor & Sign and form of the edge-level regulatory effect. \\
Accessibility / cell state & Latent context $\beta_i$ gating edge existence and regulatory role. \\
TF clamp $\doop(V_j=v)$ & $V$-node intervention fixing a regulator's expression or production/input level. \\
Binding-strength replacement & $V$-edge/mechanism intervention changing one TF--promoter channel. \\
\bottomrule
\end{tabular}
\caption{Dictionary between the SERGIO-backed GRN experiments and POSCM notation.}
\label{tab:grn-dictionary}
\end{table}

\paragraph{Shared simulation configuration.}
All quantitative results in this subsection use the full run, not the smaller execution-only smoke runs. The backend is SERGIO with regulator/target decay $0.8$, integration step $dt=0.02$ (\texttt{safety\_iter} 20, \texttt{scale\_iter} 4), and $20\times 20=400$ base cells per condition unless otherwise stated. Experiment~1 uses deterministic SERGIO dynamics plus matched measurement noise to isolate the label-swap symmetry; Experiments~2--3 use SERGIO's native stochastic expression dynamics. Each metrics file records provenance and programmatic pass/fail checks.

\paragraph{GRN Experiment 1: label-blind non-identifiability.}
Two SCMs $M$ and $M'$ differ by a hidden swap of two exchangeable regulators, \texttt{TF0} and \texttt{TF1}. In $M$, \texttt{TF0} is the activator and \texttt{TF1} is the repressor; in $M'$ the roles are swapped. Under $\beta$-blind interventions the clamp targets \texttt{TF0} and \texttt{TF1} with equal probability, so the regulator label remains hidden. The named clamp $\doop(\texttt{TF0}=v)$ is a power check that reveals the label.

\begin{table}[t]
\centering
\begin{tabular}{p{0.34\linewidth}ccc}
\toprule
Condition & KS $D$ & $p$-value & Conclusion \\
\midrule
Observational & $0.000$ & $1.00$ & Indistinguishable \\
$\beta$-blind $\doop(v=0.2)$ & $0.026$ & $0.946$ & Indistinguishable \\
$\beta$-blind $\doop(v=0.5)$ & $0.031$ & $0.830$ & Indistinguishable \\
$\beta$-blind $\doop(v=0.8)$ & $0.023$ & $0.988$ & Indistinguishable \\
$\beta$-blind $\doop(v=1.2)$ & $0.019$ & $0.999$ & Indistinguishable \\
Named $\doop(\texttt{TF0}=0.2)$ & $1.000$ & $1.1\times 10^{-239}$ & Distinguishable \\
Named $\doop(\texttt{TF0}=0.5)$ & $1.000$ & $1.1\times 10^{-239}$ & Distinguishable \\
Named $\doop(\texttt{TF0}=0.8)$ & $0.923$ & $7.1\times 10^{-184}$ & Distinguishable \\
Named $\doop(\texttt{TF0}=1.2)$ & $1.000$ & $1.1\times 10^{-239}$ & Distinguishable \\
\bottomrule
\end{tabular}
\caption{SERGIO GRN Experiment~1: the hidden activator/repressor swap is indistinguishable under observational and $\beta$-blind interventions, while the named-target control separates the two models. The blind-clamp rows pool both regulators and use 800 samples per model; observational and named-control rows use 400 samples per model.}
\label{tab:grn-exp1}
\end{table}

The blind rows pass the non-identifiability check ($p>0.05$), while the named controls pass the power check ($p<10^{-3}$). Thus, the indistinguishability is due to label-blindness rather than an insensitive test.

\paragraph{GRN Experiment 2: structure--mechanism confounding.}
The dense model has 10 regulatory sites at per-site strength $0.12$; the thinned model has 6 sites with compensated per-site strength $0.20$, so both have effective aggregate strength $1.2$. Node interventions fix the regulator's expression, whereas edge interventions replace the per-site regulatory strength.

\begin{table}[t]
\centering
\begin{tabular}{llc}
\toprule
Intervention & Condition & MMD \\
\midrule
Node-do & $v=0.2$ & $0.0184$ \\
Node-do & $v=0.5$ & $0.0268$ \\
Node-do & $v=0.8$ & $0.0103$ \\
Node-do & $v=1.2$ & $0.0074$ \\
Node-do & Mean & $0.0157$ \\
\midrule
Edge-do & $g=0.04$ & $1.1249$ \\
Edge-do & $g=0.08$ & $1.4742$ \\
Edge-do & $g=0.16$ & $1.3776$ \\
Edge-do & $g=0.32$ & $1.1464$ \\
Edge-do & Mean & $1.2808$ \\
\bottomrule
\end{tabular}
\caption{SERGIO GRN Experiment~2: calibrated dense/thinned structures match under node interventions but diverge under edge/mechanism interventions. The edge/node mean-MMD ratio is $81.5$.}
\label{tab:grn-exp2}
\end{table}

The residual node-do MMD is nonzero because SERGIO's dynamics are nonlinear and stochastic. The large edge/node ratio shows the same structure--mechanism confounding pattern as Lemma~\ref{lem:struct_mech_confounding}: setting the parent value cannot reveal how its influence is internally composed, while changing the per-site mechanism exposes the latent structural difference. SERGIO allows one interaction per regulator--target pair, so ``site count'' is encoded as effective interaction strength; this is the intended count/affinity analogue rather than literal duplicate regulatory edges.

\paragraph{GRN Experiment 3: kernel recovery.}
For the structure kernel, edges are sampled from the latent accessibility model
\[
\P(A=1\mid \beta)=\sigma(7.5(\beta-0.48)).
\]
A logistic estimator is fit from 400 binary edge samples at seven training accessibilities $\beta\in\{0.10,0.22,0.35,0.48,0.60,0.75,0.92\}$ and evaluated on five disjoint held-out accessibilities. It recovers $\sigma(7.73\beta-3.70)$, close to the ground-truth $\sigma(7.5\beta-3.60)$.

\begin{table}[t]
\centering
\begin{tabular}{cccc}
\toprule
Held-out $\beta$ & Predicted $\P(A=1)$ & Ground truth & Absolute error \\
\midrule
$0.20$ & $0.1041$ & $0.1091$ & $0.0050$ \\
$0.40$ & $0.3530$ & $0.3543$ & $0.0013$ \\
$0.55$ & $0.6351$ & $0.6283$ & $0.0068$ \\
$0.70$ & $0.8474$ & $0.8389$ & $0.0085$ \\
$0.85$ & $0.9466$ & $0.9413$ & $0.0053$ \\
\bottomrule
\end{tabular}
\caption{SERGIO GRN Experiment~3a: held-out structure-kernel recovery from binary edge samples. Held-out MAE is $0.0054$ and held-out $R^2$ is $0.99963$.}
\label{tab:grn-exp3-structure}
\end{table}

A Monte-Carlo sanity check comparing sampled edge means to the known generating probabilities gives MAE $0.0045$; this is reported only as a sampling check, not as the identification result. For the mechanism sanity check, parent clamps $v\in\{0.05,0.2,0.4,0.7,1.0,1.4\}$ recover the normalized Hill dose--response shape with $R^2=0.99987$ after shifting/rescaling the response axis. Thus, the context-to-structure kernel is genuinely identified out of sample, while the mechanism clamp-response is reported only as a shape sanity check rather than a fitted-parameter estimator.

\paragraph{Takeaway.}
Across an independent single-cell simulator, the three GRN experiments reproduce the same causal pattern as the retina experiments: label-blind hidden-context swaps remain indistinguishable, node interventions can hide structure--mechanism differences that edge/mechanism interventions expose, and context/accessibility variation identifies the structure kernel on held-out support. The SERGIO results therefore strengthen the paper's validation while preserving the distinction between identification claims and simulator-specific sanity checks.

\subsection{Tables}

\begin{table}[t]
\centering
\begin{tabular}{lll}
\toprule
Parameter & Value & Description \\
\midrule
Patch size & $30 \times 30~\mu\text{m}$ & Reduced for tractability \\
Eccentricity & $-1.2$\,mm & Mid-peripheral retina \\
Simulation time & 200\,ms & Per trial (stimulus epoch) \\
Time step & 0.1\,ms & NEURON integration step \\
Stimulus & 120\,pA flash & Full-field; rods and cones \\
Seeds & 2 & Independent network realizations \\
\bottomrule
\end{tabular}
\caption{Shared simulation parameters (unless otherwise noted).}
\label{tab:retina-shared-params}
\end{table}

\begin{table}[t]
\centering
\begin{tabular}{ccc}
\toprule
BC clamp $v$ (mV) & Mean $\Delta \bar V_{\mathrm{RGC}}$ (mV) & Std (mV) \\
\midrule
$-70$ & $-0.284$ & $0.224$ \\
$-60$ & $-0.287$ & $0.212$ \\
$-50$ & $-0.296$ & $0.199$ \\
$-40$ & $-0.095$ & $0.049$ \\
$-30$ & $+0.412$ & $0.034$ \\
$-20$ & $+0.460$ & $0.031$ \\
\bottomrule
\end{tabular}
\caption{Experiment~3 numerical voltage-sweep values used in Fig.~\ref{fig:retina-transfer-curve}.}
\label{tab:retina-exp3-transfer}
\end{table}
\end{document}